\documentclass[letterpaper, 10 pt, journal, twoside]{IEEEtran}
\newcommand{\ouralgo}{BVL}

\usepackage{xcolor}

\newcommand{\rev}[1]{#1}
\newcommand{\orgg}[1]{}
\newcommand{\revv}[1]{#1}

\newif\ifarxiv
\arxivfalse  

\newif\ifarxivfin  
\arxivfintrue   

\usepackage[normalem]{ulem}

\usepackage{multicol}
\usepackage[bookmarks=true]{hyperref}

\usepackage{amsmath,mathtools} 
\usepackage{amssymb}  
\usepackage{amsthm}  
\usepackage{mathrsfs}  

\usepackage[cal=dutchcal]{mathalfa}

\DeclareMathAlphabet{\mathcalorg}{OMS}{cmsy}{m}{n}

\usepackage{dsfont}

\usepackage{subcaption}
\newtheorem{theorem}{\bf Theorem}
\newtheorem{lemma}{\bf Lemma}

\usepackage{algorithm}
\usepackage[noend]{algpseudocode}
\algnewcommand\algorithmicinput{\;\;\textbf{input}:}
\algnewcommand\algorithmicoutput{\;\;\textbf{output}:}
\algnewcommand\algorithmicparameters{\;\;\textbf{parameters}:}
\algnewcommand\algorithmicglobalvariables{\;\;\textbf{global variables}:}
\algnewcommand\Input{\item[\algorithmicinput]}
\algnewcommand\Output{\item[\algorithmicoutput]}
\algnewcommand\Parameters{\item[\algorithmicparameters]}
\algnewcommand\GlobalVariables{\item[\algorithmicglobalvariables]}

\newcommand{\tr}{\text{tr}}
\newcommand{\pr}[1]{\textbf{#1}:}
\newcommand{\diag}{\text{diag}}

\newcommand{\argmin}{\mathop{\mathrm{argmin}}}

\newcommand{\rnp}{\textit{RNP}}
\newcommand{\infotrap}{\textit{InfoTrap}}
\newcommand{\obswall}{\textit{ObsWall}}
\newcommand{\forest}{\textit{Forest}}

\usepackage{scalerel}
\newcommand\smallplus{{\scaleobj{0.8}{+}}}

\newcommand\smallminus{{\scaleobj{0.8}{-}}}

\newcommand*\xbar[1]{%
	\hbox{%
		\vbox{%
			\hrule height 0.5pt 
			\kern0.3ex
			\hbox{%
				\kern-0.0em
				\ensuremath{#1}%
				\kern-0.15em
			}%
		}%
	}%
} 

\graphicspath{{./3_FIRMCP_WAFR/}}

\ifCLASSINFOpdf
\else
\fi

\begin{document}
%
\ifarxivfin
\title{
	Bi-directional Value Learning for Risk-aware Planning Under Uncertainty: Extended Version
}
\fi
\author{Sung-Kyun Kim, Rohan Thakker, and Ali-akbar Agha-mohammadi%
	\thanks{Manuscript received September 10, 2018; accepted January 30, 2019. Date of publication; date of current version. This letter was recommended for pub- lication by Associate Editor A. Faust and Editor N. Amato upon evaluation of the reviewers’ comments.
		This research is partially carried out at the Jet Propulsion Laboratory  (JPL)  and  the  California  Institute  of  Technology  (Caltech)  under  a  contract  with  the  NASA  and  funded  through the President’s and Director’s Fund 105275-18AW0056.
		\textit{(Corresponding author: Sung-Kyun Kim.)}%
	} 
	\thanks{S.-K. Kim is with the Robotics Institute, Carnegie Mellon University, 5000 Forbes Ave, Pittsburgh, PA 15215 USA (e-mail: kimsk@cs.cmu.edu).}%
	\thanks{R. Thakker and A.-A. Agha-Mohammandi are with Jet Propulsion Laboratory, Pasadena, CA 91101 USA (e-mail: rohan.a.thakker@jpl.nasa.gov; aliagha@jpl.nasa.gov).}%
	\thanks{Digital Object Identifier 10.1109/LRA.2019.2903259}
}
\markboth{IEEE Robotics and Automation Letters. Preprint Version. Accepted January 2019}
{Kim \MakeLowercase{\textit{et al.}}: Bi-directional Value Learning for Risk-aware Planning Under Uncertainty}

%




\maketitle


\begin{abstract}
	Decision-making under uncertainty is a crucial ability for autonomous systems.
	In its most general form, this problem can be formulated as a Partially Observable Markov Decision Process (POMDP). The solution policy of a POMDP can be implicitly encoded as a value function.
	In partially observable settings, the value function is typically learned via forward simulation of the system evolution. Focusing on accurate and long-range risk assessment, we propose a novel method, where the value function is learned in different phases via a 
	bi-directional search in belief space. A backward value learning process provides a long-range and risk-aware base policy. A forward value learning process ensures local optimality and updates the policy via forward simulations. We consider a class of scalable and continuous-space rover navigation problems (RNP) to assess the safety, scalability, and optimality of the proposed algorithm. The results demonstrate the capabilities of the proposed algorithm in evaluating long-range risk/safety of the planner while addressing continuous problems with long planning horizons.
\end{abstract}


%

\section{Introduction}
\IEEEPARstart{C}{onsider} a scenario where an autonomous mobile robot (e.g., a rover or flying drone) needs to navigate through an obstacle-laden environment under both motion and sensing uncertainty. In spite of these uncertainties, the robot needs to guarantee safety and reduce the risk of collision with obstacles at all times. This, in particular, is a challenge for safety-critical systems and fast moving robots as the vehicle traverses long distances in a short time horizon. Hence, ensuring system's safety requires risk prediction over long horizons.

The above-mentioned problem is an instance of general problem of decision-making under uncertainty in the presence of risk and constraints, which has applications in different mobile robot navigation scenarios. 
This problem in its most general and principled form can be formulated as a Partially Observable Markov Decision Process (POMDP) \cite{Kaelbling98,kochenderfer2015decision}. In particular, in this work, we focus on a challenging class of POMDPs, here referred to as RAL-POMDPs (Risk-Averse, Long-range POMDPs). A RAL-POMDP reflects some of challenges encountered in physical robot navigation problems, and is characterized with the following features:
\begin{enumerate}
    \item Long planning horizons (beyond $10^4$ steps) without discounting cost over time, i.e., safety is equally critical throughout the plan. In RAL-POMDP, the termination of the planning problem is dictated by reaching the goal (terminal) state rather than reaching a finite planning horizon.
    \item RAL-POMDP is defined via high-fidelity continuous state, action, and observation models.
    \item RAL-POMDP incorporates computationally expensive costs and constraints such as collision checking.
	\item RAL-POMDP requires quick policy updates to cope with local changes in the risk regions during execution.
\end{enumerate}

\ifarxiv
\else  

\begin{figure}[t]
\rev{
\begin{center}
	\begin{minipage}{0.29\columnwidth}
		\centering
		\includegraphics[width=\textwidth]{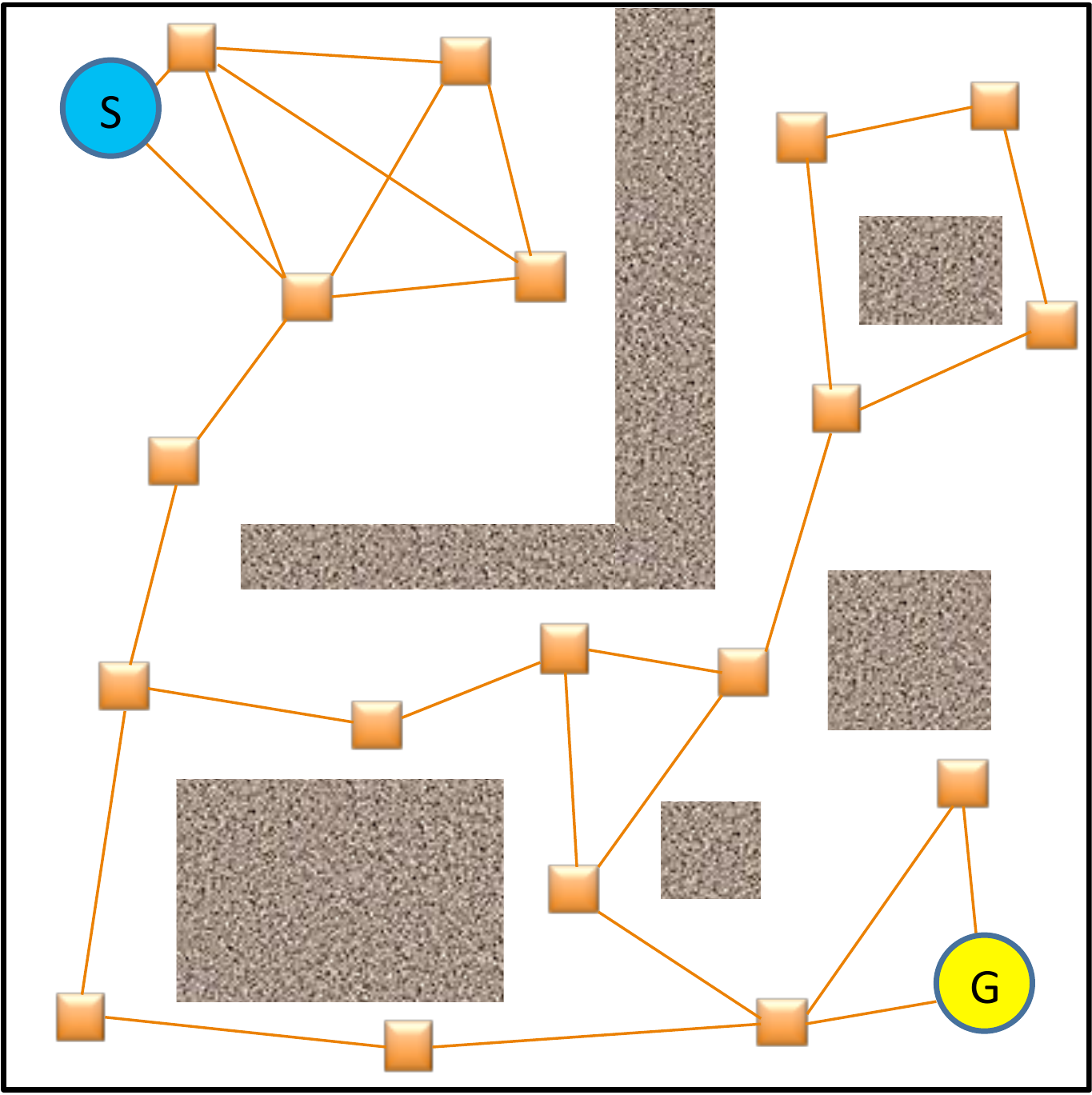} \\
		\includegraphics[width=\textwidth]{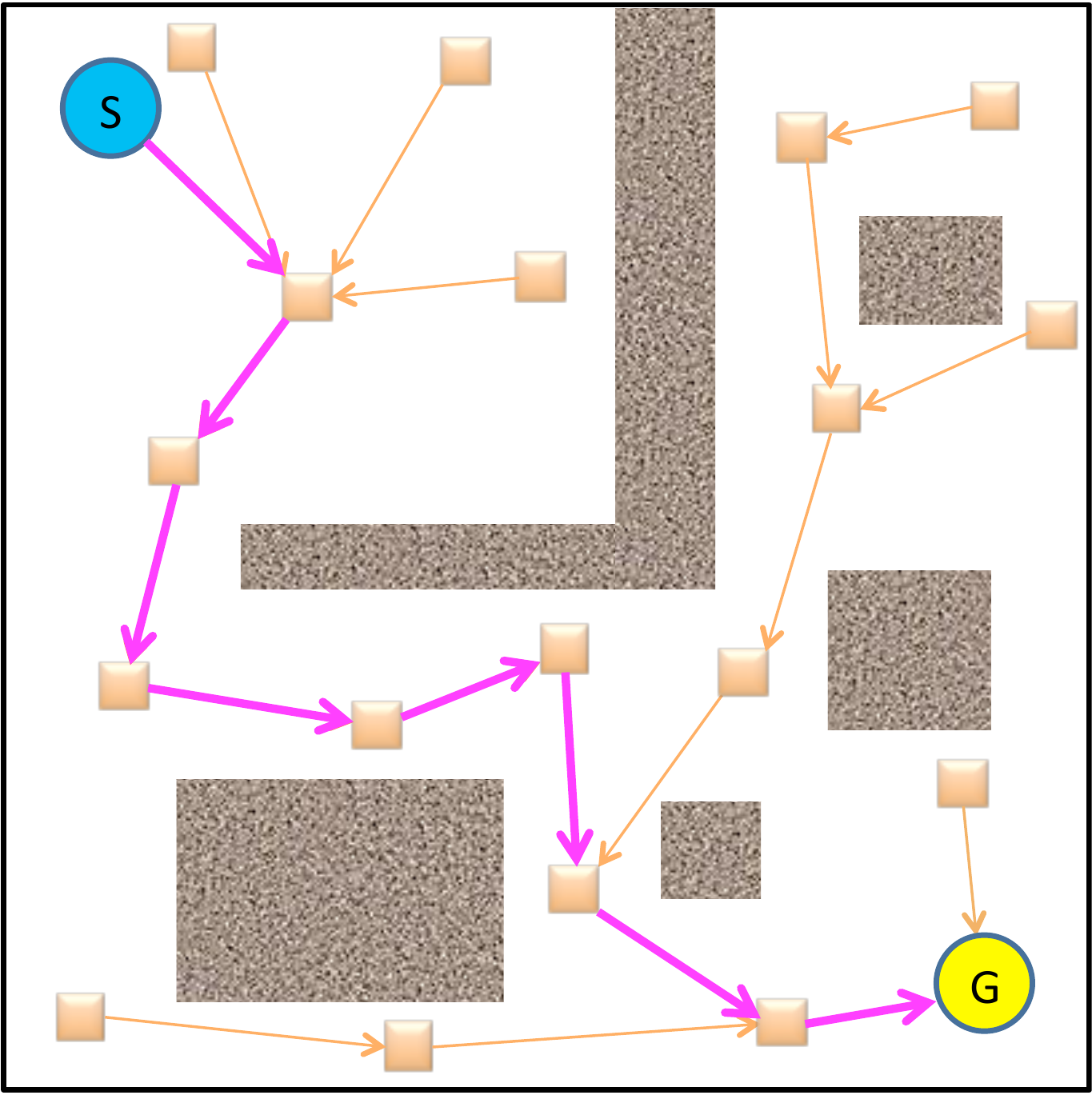} \\
		\footnotesize{\centering{(a) Backward long-range solver}}
	\end{minipage}
	\begin{minipage}{0.29\columnwidth}
		\centering
		\includegraphics[width=\textwidth]{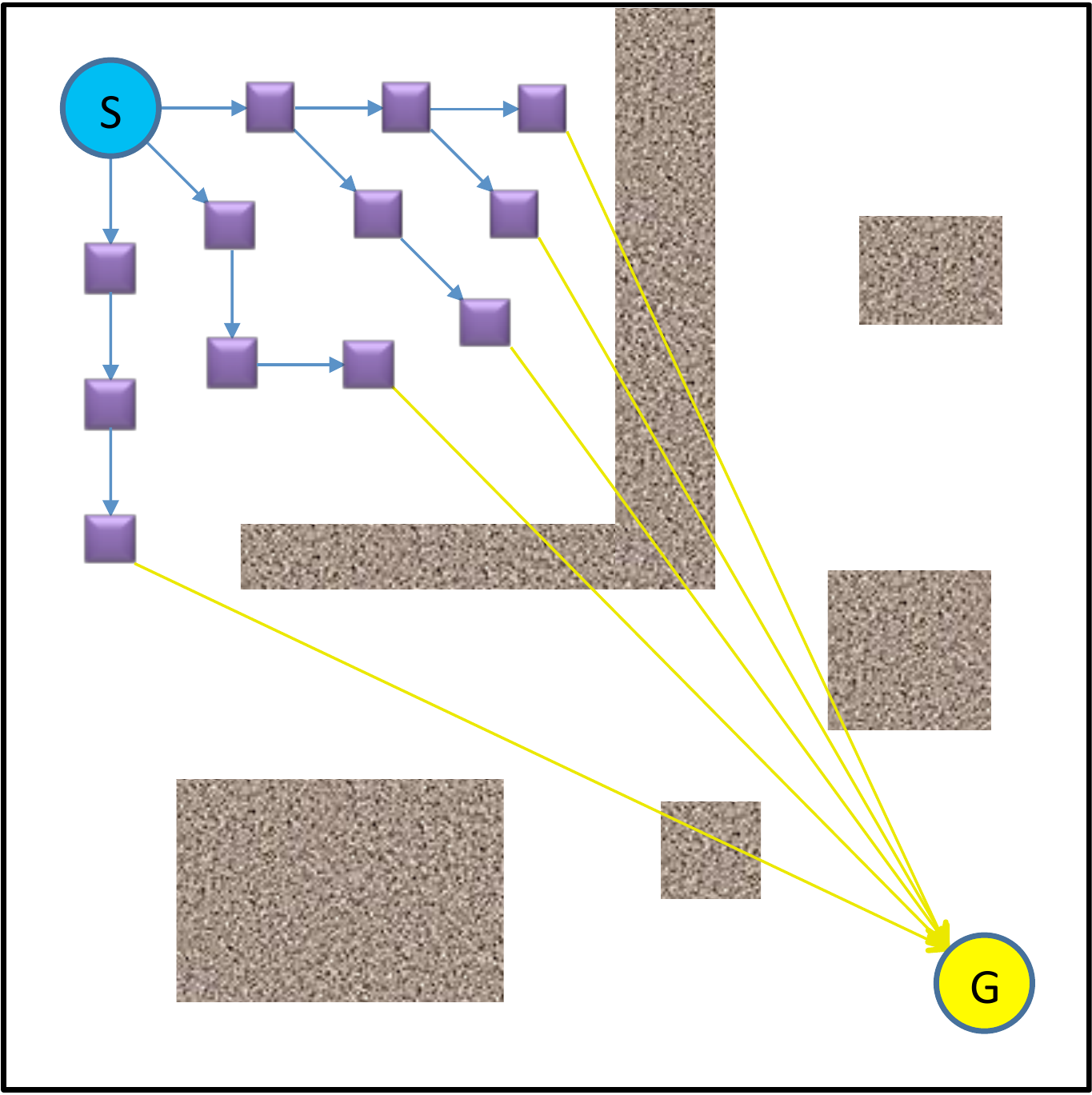} \\
		\includegraphics[width=\textwidth]{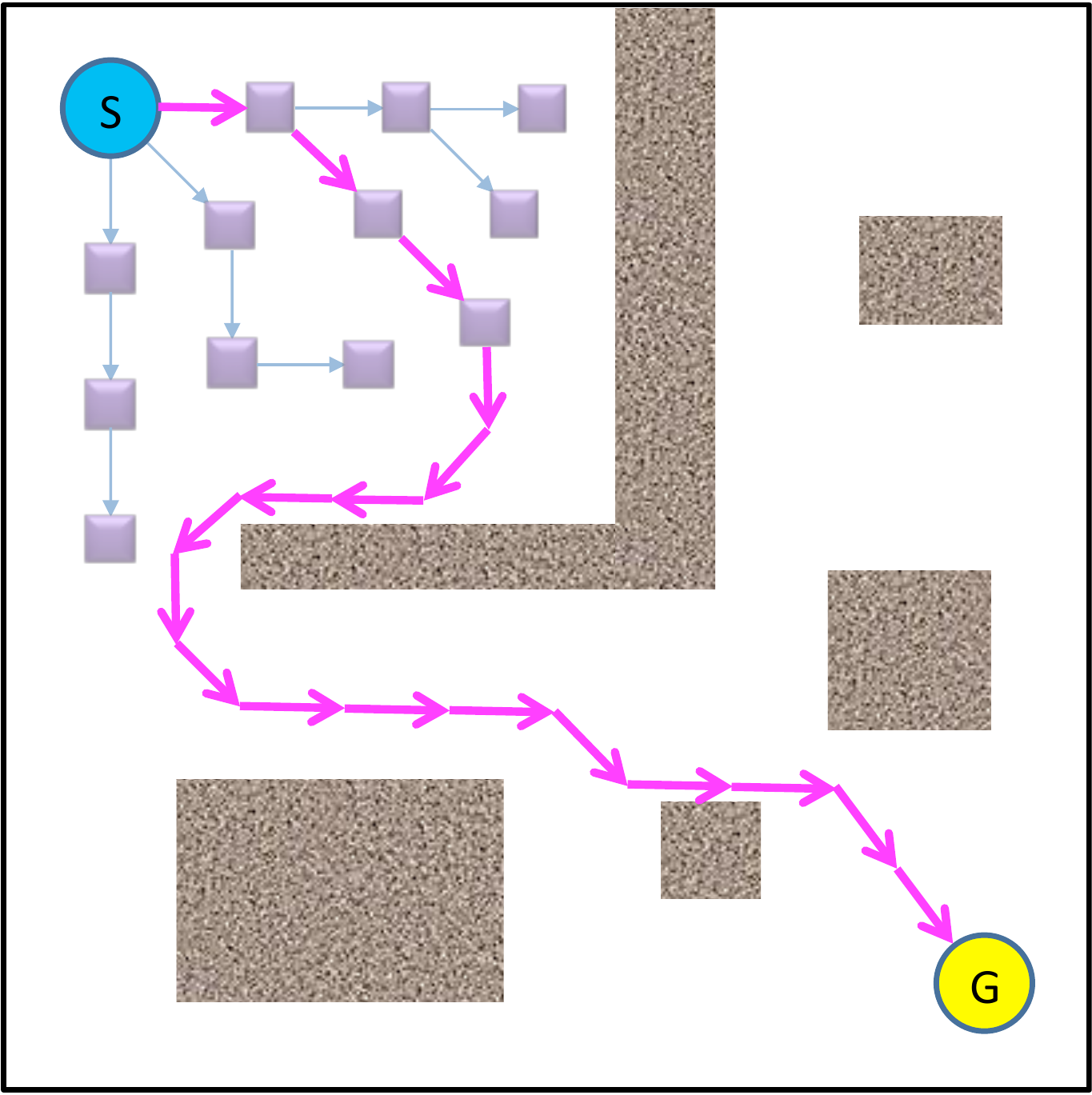} \\
		\footnotesize{\centering{(b) Forward short-range solver}}
	\end{minipage}
	\begin{minipage}{0.29\columnwidth}
		\centering
		\includegraphics[width=\textwidth]{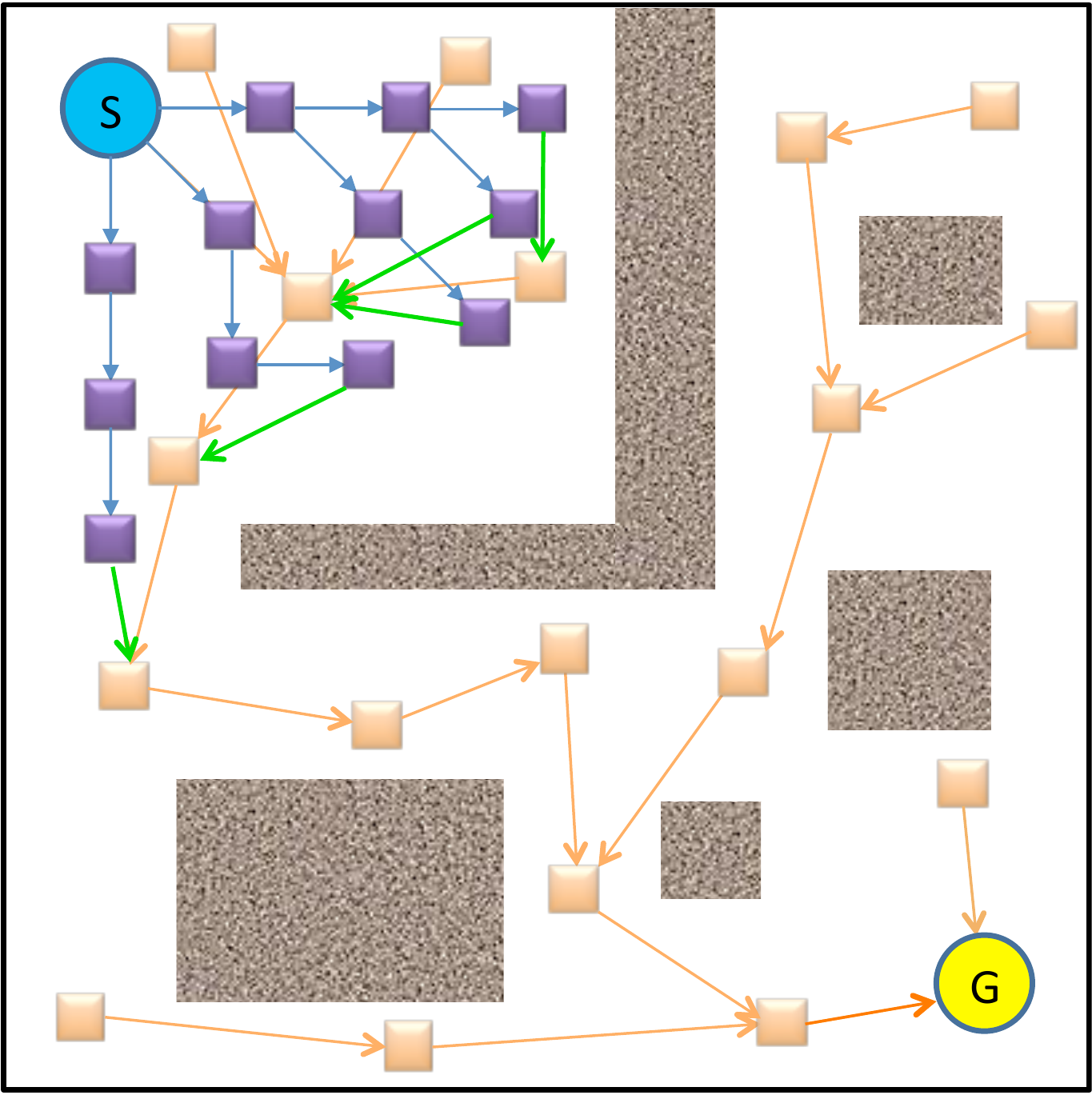} \\
		\includegraphics[width=\textwidth]{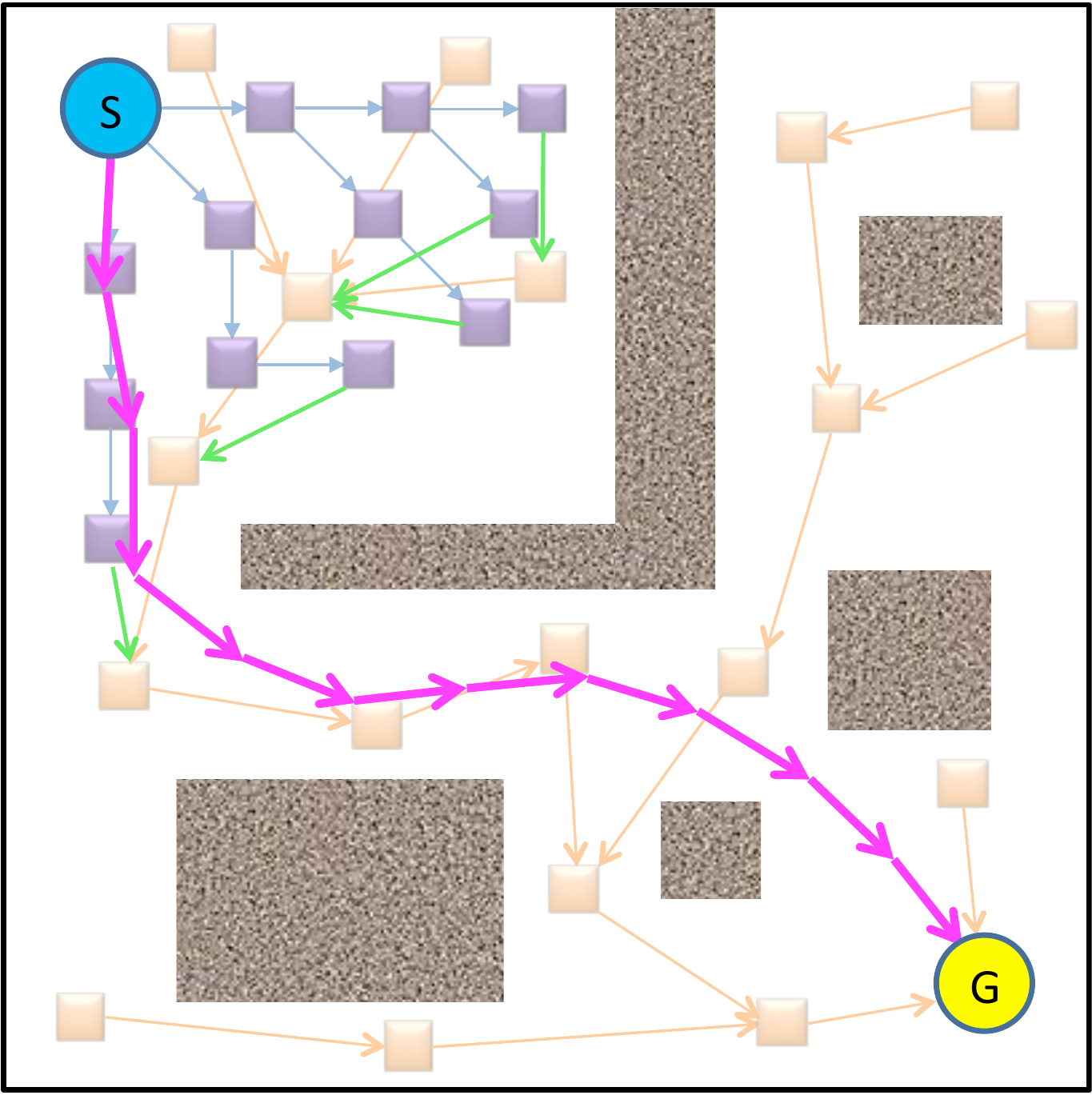} \\
		\footnotesize{\centering{(c) Combined bi-directional solver}}
		\vspace{0.2em}
	\end{minipage}
	\caption{\rev{
			Illustration of the planning procedure (top row) and the most-likely execution path (magenta arrows in the bottom row) of each solver.
			\,(a) Backward long-range solver constructs a belief graph (orange lines) and solves for an approximate global policy (orange~arrows) \textit{to the goal}.
			Its solution is suboptimal due to the finite number of sampling.
			\,(b) Forward short-range solver constructs a belief tree (blue arrows) \textit{from the current belief} up to its finite horizon and uses heuristic estimates (yellow arrows) of costs-to-go to generate a locally near-optimal policy.
			Its solution may suffer from local minima due to its finite planning horizon.
			\,(c) Bi-directional solver combines a forward short-range solver and a backward long-range solver by \textit{bridging} (green arrows) the forward belief tree (blue arrows) to the approximate global policy (orange arrows). 
			Thus, it can provide a solution with improved scalability and performance.
		}}
	\label{fig:overview}
\end{center}
}  
\vspace{-15pt}
\end{figure}

\fi  

In recent years, value learning in partially observable settings has seen impressive advances in terms of the complexity and size of solved problems.
\ifarxiv
\else
There are two major classes of POMDP solvers (see Fig.~\ref{fig:overview}).
\fi
The first class is forward search methods \cite{Kurniawati08-SARSOP,Pineau03,silver2010monte,gelly2011monte}. Methods in this class (offline and online variants) typically rely on forward simulations to search the reachable belief space from a given starting belief and learn the value function. POMCP (Partially Observable Monte Carlo Planning) \cite{silver2010monte}, DESPOT \cite{somani2013despot}, and ABT \cite{kurniawati2016online} are a few examples of methods in this class that can efficiently learn and update the policy while executing a plan using Monte Carlo simulation.

The second class is approximate long-range solvers such as FIRM (Feedback-based Information RoadMap) \cite{Prentice09,Ali14-IJRR}. These methods typically address continuous POMDPs but under the Gaussian assumption. They typically rely on graph construction and feedback controllers to solve larger problems. Through offline planning, they can learn an approximate value function on the representative (sampled) graph. 


The features of RAL-POMDP problems make them a challenging class of POMDPs for above-mentioned solvers. Forward search-based methods typically require cost discounting and a limited horizon (shorter than 100 steps) to be able to handle the planning problem. Also, they typically require at least one of the state, action, or observation space to be discrete. Continuous approximate long-range methods suffer from suboptimality since actions are generated based on a finite number of local controllers due to the underlying sparse sampling-based structure. 

This work addresses RAL-POMDP problems induced by fast-moving robot navigation in safety-critical scenarios. In such systems, several seconds of operation can translate to thousands of decision making steps. The main objective of this work is to provide probabilistic safety guarantees for the long-horizon decision making process (beyond thousands of steps). The second objective of this work is to generate solutions for RAL-POMDPs that are closer to the globally optimal solution compared to the state-of-the-art methods. In parallel to these objectives, we intend to satisfy other requirements of the RAL-POMDP such as incorporating high-fidelity continuous dynamics and sensor models into the planning.

In this paper, we propose Bi-directional Value Learning (\ouralgo{}) method, a POMDP solver that searches the belief space and learns the value function in a bi-directional manner. In the one thread (can be performed offline) we learn a risk-aware approximate value function backwards from the goal state toward the starting point.
In the second thread (performed online), we expand a forward search tree from the start toward the goal. 
\ouralgo{} significantly improves the performance (optimality) of the backward search methods by locally updating the policy through rapid online forward search during the actual execution.
\ouralgo{} also enhances the probabilistic guarantees on system's safety by performing computationally intensive processes, such as collision checking, over long planning horizons in the offline phase.

In Section II, we go over the formal definition of POMDP problems and explain more details about RAL-POMDP problems.
In Section III, we present the overall framework of \ouralgo{} and its concrete instance based on POMCP \cite{silver2010monte} and FIRM \cite{Ali14-IJRR}.
Section IV provides various simulation experiments to validate the \ouralgo{} method,
and Section V concludes this paper.

\section{Preliminaries}


\subsection{POMDP Problems} \label{ssec:pomdp}
Let us denote the system state, action, and observation at the $k$-th time step by $x_k \in \mathbb{X}$, $u_k \in \mathbb{U}$, and $z_k \in \mathbb{Z}$. The motion model $f$ and observation model $h$ can be written as:
\begin{gather}
	x_{k+1} = f(x_k, u_k, w_k), ~ w_k \sim p(w_k|x_k, u_k) \\
	z_{k} = h(x_{k}, v_{k}), ~ v_{k} \sim p(v_{k}|x_{k})
\end{gather}
where $w_k$ and $v_k$ denote the motion and sensing noises.

A belief state $b \in \mathbb{B}$ is a posterior distribution over all possible states given the past actions and observations $b_{k} = p(x_{k} | z_{0:k}, u_{0:k-1})$, which can be updated recursively via Bayesian inference:
\begin{align}
	b_{k+1} &= \tau(b_k, u_k, z_{k+1})
\end{align}

A policy $\pi : \mathbb{B} \rightarrow \mathbb{U}$ maps each belief state $b$ to a desirable action $u$.
Denoting the one-step cost function as $c(b, u) \in \mathbb{R}_{> 0}$,
the value function (or more precisely, the expected cost-to-go function) under policy $\pi$ is defined as follows.
\begin{align}
  J(b; \pi) 
  & = \mathbb{E} \left[ \sum_{k=0}^\infty \gamma^k c(b_k, \pi(b_k)) \right]
	\label{eq:trueJ}
	\\
	& = c(b, \pi(b)) + \gamma \sum_{b' \in \mathbb{B}} p(b' | b, \pi(b)) J(b'; \pi)
	\label{eq:trueJrecursive}
\end{align}
where $b_0 = b$,
$\gamma \in (0, 1]$ is a discount factor that reduces the effect of later costs,
and $p(b' | b, u)$ is the transition probability from $b$ to $b'$ under action $u$.
Equation (\ref{eq:trueJrecursive}) in a recursive form is called a Bellman equation.
\rev{
	It is also convenient to define an intermediate belief-action function, or Q-value, as $Q(b, u; \pi) = c(b, u) + \gamma \sum_{b' \in \mathbb{B}} p(b' | b, u) J(b'; \pi)$,
	such that
	\begin{align}
		J(b; \pi) = \min_{u \in \mathbb{U}} Q(b, u; \pi)
		\label{eq:minq}
	\end{align}
}  

A POMDP problem can then be cast as finding the optimal value and policy.
\begin{align}
	\pi^*(b) = \argmin_{\pi\in\Pi} J(b; \pi),~\forall b \in \mathbb{B}
	\label{eq:optonline}
\end{align}




\subsection{RAL-POMDP} \label{ssec:ral-pomdp}
In this work, we focus on a RAL-POMDP as a special case of the above-mentioned POMDP problem. Formally, in a RAL-POMDP, $\mathbb{X}$, $\mathbb{U}$, and $\mathbb{Z}$ are continuous spaces, and $f$ and $h$ represent locally differentiable nonlinear mappings.
There exists a goal termination set $B^{goal}\subset\mathbb{B}$ such that $J(b_g)=0$ for $\forall b_g \in B^{goal}$.
There also exists a failure termination set $F\subset \mathbb{B}$ which represents the risk region (e.g., obstacles in robot motion planning)
such that $J(b_f) \to \infty$ for $\forall b_f \in F$.
As the risk is critical throughout the plan, a RAL-POMDP does not allow cost discounting, i.e., $\gamma=1$.


\rev{

In our risk metric discussion, we follow definitions in \cite{ruszczynski2010risk,majumdar2017should}. Accordingly, our risk metric falls in the category of risk for sequential decision making with deterministic policies, satisfying 
\textit{time-consistency} (see \cite{ruszczynski2010risk,majumdar2017should} for details).  
Specifically, we formalize the risk by compounding the failure probability, $p(F | b, \pi(b)) = \sum_{b_f \in F} p(b_f | b, \pi(b))$, of each action along the sequence.
Accordingly, the risk metric of a policy given a belief $b_0$ is measured as follows.
\begin{align}
	\rho(b_0; \pi) = 1 - \mathbb{E} \left[ \prod_{k=0}^{\infty} (1 - p(F | b_k, \pi(b_k) ) \right]
	\label{eq:riskdef}
\end{align}
The second term on the right-hand side is the expected probability to reach the goal without hitting the risk region.
Note that $\rho(b_g; \pi) = 0$ for $\forall b_g\in B^{goal}$ and  $\rho(b_f; \pi) = 1$ for $\forall b_f \in F$ for $\forall \pi \in \Pi$.
It can be rewritten in a recursive form:
\ifarxivfin
\begin{align}
	\rho(b; \pi) &= 1 - \sum_{b' \in \mathbb{B}} p(b' | b, \pi(b)) (1  - \rho(b'; \pi))
	\nonumber \\
	&= \sum_{b' \in \mathbb{B}} p(b' | b, \pi(b)) \rho(b'; \pi)
	\label{eq:risk}
\end{align}
\fi
}

\ifarxivfin

Now we show that in RAL-POMDPs where $J(b_f) = J^F \to \infty$ for $\forall b_f \in F$, the optimal policy $\pi^*$ in Eq. (\ref{eq:optonline}) also minimizes $\rho(b; \pi^*)$ in Eq. (\ref{eq:risk}) for $\forall b \in \mathbb{B}$.


\begin{lemma}
	In RAL-POMDPs where $J^F \to \infty$ and $\gamma = 1$, the following is satified for $\forall b \in \mathbb{B}$.
	\begin{align}
		\rho(b; \pi) = \lim_{J^F\to\infty} \frac{J(b;\pi)}{J^F}
		\label{eq:equiv}
	\end{align}
	\label{lemma:equiv}
\end{lemma}	
%
\begin{proof}
We prove this by \textit{backward induction}. 

Consider the terminal beliefs first.
Trivially, from Eq. (\ref{eq:trueJ}) and Eq. (\ref{eq:riskdef}),
$\rho(b_g; \pi) = \lim_{J^F\to\infty} \frac{J(b_g;\pi)}{J^F} = 0$ for $\forall b_g \in B^{goal}$,
and $\rho(b_f; \pi) = \lim_{J^F\to\infty} \frac{J(b_f;\pi)}{J^F} = 1$ for $\forall b_f \in F$.
Thus, Eq. (\ref{eq:equiv}) is satified for terminal beliefs.

Next, consider a belief such that its every successor 
is either $b_g \in B^{goal}$ or $b_f \in F$, i.e., $\sum_{b_g \in B^{goal}} p(b_g|b, \pi(b)) + \sum_{b_f \in F} p(b_f|b, \pi(b)) = 1$.
Then from Eq. (\ref{eq:risk}),
\begin{align}
	\rho(b; \pi) &= \sum_{b_f \in F} p(b_f|b, \pi(b)) \cdot 1
\end{align}
and from Eq. (\ref{eq:trueJrecursive}) with $\gamma=1$ we have:
\begin{align}
	\lim_{J^F\to\infty} \frac{J(b_f;\pi)}{J^F} &= \lim_{J^F\to\infty} \frac{\sum_{b_f \in F}p(b_f | b, \pi(b)) \cdot J^F}{J^F}
\end{align}
Thus, all such belief $b$ satisfies Eq. (\ref{eq:equiv}).


Now we consider a belief $b$ such that its all successors $\{b' | b, \pi(b)\}$ satisfy Eq. (\ref{eq:equiv}).
By injecting Eq. (\ref{eq:equiv}) for the successors into Eq. (\ref{eq:risk}), 
\begin{align}
	\rho(b; \pi) 
	&= \sum_{b' \in \mathbb{B}} \lim_{J^F\to\infty} p(b' | b, \pi(b)) \frac{J(b';\pi)}{J^F}
\end{align}
By dividing Eq. (\ref{eq:trueJrecursive}) by $J^F$,
\begin{align}
	\lim_{J^F\to\infty} \frac{J(b; \pi)}{J^F} &= \lim_{J^F\to\infty} \frac{1}{J^F} \sum_{b' \in \mathbb{B}} p(b' | b, \pi(b)) J(b'; \pi)
\end{align}
Thus, it satisfies Eq. (\ref{eq:equiv}).

Finally, by backward induction, Eq. (\ref{eq:equiv}) is satisfied for $\forall b \in \mathbb{B}$ in RAL-POMDPs.
\end{proof}



\begin{theorem}
	In RAL-POMDPs where $J^F \to \infty$ and $\gamma = 1$, the optimal policy $\pi^*$ that minimizes $J(b; \pi^*)$ also minimizes $\rho(b; \pi^*)$ for $\forall b \in \mathbb{B}$.
\end{theorem}

\begin{proof}
    First, we can rewrite Eq. (\ref{eq:optonline}) as follows for RAL-POMDPs where $J^F \to \infty$.
	\begin{align}
		\pi^*(b) = \argmin_{\pi \in \Pi} \lim_{J^F \to \infty} J(b;\pi),~\forall b \in \mathbb{B}
	\end{align}
	By dividing the objective function in Eq. (\ref{eq:optonline}) by a constant $J^F$, we have: 
	\begin{align}
		\pi^*(b) = \argmin_{\pi \in \Pi} \lim_{J^F \to \infty} \frac{J(b; \pi)}{J^F},~\forall b \in \mathbb{B}
	\end{align}
	Then by Lemma~\ref{lemma:equiv}, we prove the theorem.
	\begin{align}
		\pi^*(b) = \argmin_{\pi \in \Pi} \rho(b; \pi),~\forall b \in \mathbb{B}
	\end{align}
\end{proof}

\fi  

\section{Bi-directional Value Learning (\ouralgo{})}


	\subsection{Overall Framework}


	In this section, we provide the framework of \ouralgo{}, the proposed \textit{bi-directional long-short-range} POMDP solver, and its concrete instance based on POMCP \cite{silver2010monte} and FIRM \cite{Ali14-IJRR}.
	Figure~\ref{fig:overview} conceptually shows how the combination of the forward short-range and backward long-range planner works.
	The short-range solver relies on the knowledge of the initial belief and is limited to its reachable belief subspace. It can find a locally near-optimal policy but may get stuck in local minima in the global perspective.
	The long-range solver can provide a global policy to reach to the goal, but it only considers (sampled) subspace, which results in the solution suboptimality.
	The main idea of \ouralgo{} is to develop a bridging scheme between these two approaches to take advantage of both solvers while alleviating their drawbacks.

	\ifarxiv
\else  

	In \ouralgo{}, the optimization in Eq. (\ref{eq:optonline}) is decomposed as: 
	\begin{align}
		\!\!\!\!\pi(\cdot) \!=\! \argmin_{\Pi} \mathbb{E} \left[ C^{sr}(b^{sr}, \pi) 
		\!+\! C^{br}(b^{br}, \pi) 
		\!+\! C^{lr}(b^{lr}, \pi)
		\right]
		\label{eq:abstract}
	\end{align}
	First term, $C^{sr}$, is the cost learned by the short-range planner. $C^{lr}$ is the cost computed by the long-range planner. $C^{br}$ is the cost learned by the bridge planner that connects the short-range policy to the long-range policy.

	More concretely, Eq. (\ref{eq:abstract}) can be rewritten as follows for the instance of \ouralgo{} based on POMCP and FIRM.
	\begin{align}
		\pi(\cdot) =& \argmin_\Pi \mathbb{E}\! \left[ \sum_{k=0}^{K^{sr}\!-1}\!\! c(b_k, \pi_k(b_k)) \right.
			\nonumber
			\\
			& \left. + \!\!\sum_{k=K^{sr}}^{K^{sr}\!+K^{br}\!-1} \!\!c(b_k, \pi_{(K^{sr}\!-1)}(b_k))
		 	\!+ \!\tilde{J}^g(B^{j}\!) \right]
		\label{eq:firmcp}
	\end{align}
	where $K^{sr}$ is the fixed horizon of the short-range planner, 
	and $K^{br}$ is a varying horizon of a bridge planner 
	that takes the belief $b_{(K^{sr}+K^{br})}$ (at the end of bridging) to a node $B^{j}\!$ of the global long-range policy.
	$\tilde{J}^g(B^{j^\smallplus})$ denotes the approximate estimate of the cost-to-go of $B^{j}\!$ computed offline.

	\fi  

	In the following sections, we will discuss this decomposition in more detail using concrete instantiations of the short-range and long-range planners.

	\subsection{Long-range Global Planner}
    For our long-range global policy, we utilize FIRM (Feedback-based Information Roadmap) \cite{Ali14-IJRR}.
	FIRM is an offline, approximate long-range planner. FIRM locally approximates the system model with linear Gaussian models and generates a graph (see Fig.~\ref{fig:overview}-top) of Gaussian distributions in the belief space. 
	We formally describe the offline planning here (Algorithm~\ref{alg:firm}):
	Let us define the $i$-th FIRM node $B^i$ as a set of belief states near a center belief $b_c^i \equiv (v^i, P_c^i)$, where $v^i$ is a sampled point in state space and $P_c^i$ is the node covariance.
	\begin{align}
		B^i &= \{ b : || b - b_c^i || \leq \epsilon \}
			\label{eq:isReached}
	\end{align}
	$\epsilon$ is the node size and
	$\mathbb{V}^g = \{ B^i \}$ is the set of all FIRM nodes.
	
	For a pair of neighboring nodes $B^i$ and $B^j$, 
	a local closed-loop controller $\mu^{ij}: \mathbb{B} \rightarrow \mathbb{U}$ can be designed (e.g., Linear Quadratic Gaussian controllers) that can steer the belief from $B^i$ to $B^j$.
	We denote the set of all local controllers as $\mathbb{M}^g = \{ \mu^{ij} \}$ and the set of all local controllers originated from $B^i$ as $\mathbb{M}(i) \subset \mathbb{M}^g$. After graph construction, FIRM associates a cost function to each edge by simulating the local controller, $\mu^{ij}$ from $B^i$ to $B^j$.
	\begin{align}
		\tilde{C}^g(B^i, \mu^{ij}) = \sum_{k=0}^{\mathcal{K}^{ij}} c(b_{k}, \mu^{ij}(b_{k}))
		\label{eq:nominalCost}
	\end{align}
	where $b_0 = b_c^i$.
	$\mathcal{K}^{ij}$ is the number of time steps it takes for controller $\mu^{ij}$ to take belief $b_k$ from $B^i$ to $B^j$.


	A policy over FIRM graph is a mapping from nodes to edges, i.e., $\tilde{\pi}^g: \mathbb{V}^g \rightarrow \mathbb{M}^g$. Approximate cost-to-go for a given $\tilde{\pi}^g$ can be computed as follows.
	\begin{align}
		\tilde{J}^g(B^i; \tilde{\pi}^g) 
		& = \mathbb{E} \left[ \sum_{k=0}^\infty \tilde{C}^g (B_k, \tilde{\pi}^g(B_k)) \right]
		\label{eq:approxJ}
	\end{align}
	where $B_0 = B^i$. We denote by $\mathbb{N}(B^i)$ the set of neighbor FIRM nodes of $B^i$.
	Equation (\ref{eq:approxJ}) can also be rewritten in a recursive form as follows.
	\begin{align}
		& \tilde{J}^g(B^i; \tilde{\pi}^g) =\; \tilde{C}^g(B^i, \tilde{\pi}^g(B^i))  \nonumber \\
		& \quad\quad\;\;\;\, + \! \sum_{B^j \in \mathbb{N}(B^i)} \mathbf{P}^g(B^j | B^i, \tilde{\pi}^g(B^i)) \tilde{J}^g(B^j; \tilde{\pi}^g)
		\label{eq:bellmanfirm}
	\end{align}
	where $\mathbf{P}^g(B^j | B^i, \tilde{\pi}^g(B^i))$ is the transition probability from $B^i$ to $B^j$ under $\mu^{ij}=\tilde{\pi}^g(B^i)$.
	Note that since the transition probability is usually expensive to compute, approximation methods, such as Monte Carlo simulation, are being used.
	
	Then the following optimization problem is solved by value iteration to find a global policy for the sampled subspace.
	\begin{align}
		\tilde{\pi}^{g^*}(\cdot) &= \argmin_{\tilde{\Pi}^g} \tilde{J}^g(B_k; \tilde{\pi}^g)
	\end{align}
	\vspace{-20pt}

\begin{algorithm}[t]
	{\fontsize{8.5pt}{9.8pt}\selectfont
	\caption{\textsc{OfflinePlanning}()}
		\label{alg:firm}
		\begin{algorithmic}[1]

			\Input
				\ {
					\par \noindent $\mathbb{X}_{free}$: free space map
					\par \noindent $B^{goal}$: goal belief node
				}

			\Output
				\ {
					\par \noindent $(\mathbb{V}^g, \mathbb{M}^g)$: FIRM graph
					\par \noindent $\tilde{J}^g(B^j)$: cost-to-go for all  $B^j \in \mathbb{V}^g$
				}

			\Procedure{OfflinePlanning$()$}{}

				\State $\mathbb{V}^g \gets \{B^{goal}\}, ~ \mathbb{M}^g \gets \emptyset $
				
				\State Sample PRM nodes $\mathcal{V} = \{v^j\}$ s.t. $v^j \in \mathbb{X}_{free}$
				using \cite{Kavraki96}
				\State {Construct PRM edges $\mathcal{E} = \{e^{ij}\}$, where $e^{ij}$ is an edge from $v^i$ \unskip\parfillskip 0pt \par}
			 		\par \noindent \quad\,\,	to $v^j \in \mathbb{N}(v^i)$
				\ForAll{$v^i \in \mathcal{V}$}
					\State Design a controller and construct a FIRM node $B^i$ using \cite{Ali14-IJRR}
					\State $\mathbb{V}^g \gets \mathbb{V}^g\cup\{B^j\}$
				\EndFor
				\ForAll{$e^{ij} \in \mathcal{E}$}
					\State Design the controller $\mu^{ij}$ along $e^{ij}$ using \cite{Ali14-IJRR}
					\State $\mathbb{M}^g \gets \mathbb{M}^g\cup\{\mu^{ij}\} $
					\State {Compute transition cost $\tilde{C}^g(B^i, \mu^{ij})$ and transition proba- \unskip\parfillskip 0pt \par}
						\par \noindent \quad\,\quad\;\; bility $\mathbf{P}^g(B^j | B^i, \mu^{ij})$
				\EndFor
				\State $(\tilde{J}^g, \tilde{\pi}^g) \gets \textsc{ValueIteration} (\mathbb{V}^g,\mathbb{M}^g,\tilde{C}^g,\mathbf{P}^g)$

				\State \Return $(\tilde{J}^g(B^j), \tilde{\pi}^g(B^j))$ for all $B^j \in \mathbb{V}^g$

			\EndProcedure
		\end{algorithmic}
	}
	\end{algorithm}

\orgg{
	\subsection{Short-range Local Planner}
	\label{ssec:pomcp}

	To find a locally near-optimal policy in an online manner, we adopt POMCP (Partially Observable Monte Carlo Planning) \cite{silver2010monte} for this \ouralgo{} instance.
	POMCP is an online POMDP solver that uses Monte Carlo Tree Search (MCTS) in belief space and particle representation for belief states.

	POMCP's action selection during Monte Carlo simulation is governed by two policies: a tree policy within the constructed belief tree, and a rollout policy out of the tree but up to the finite discount horizon.
	
	The tree policy selects an action based on PO-UCT algorithm as follows.
	\begin{align}
    u^* &= \argmin_{u \in \mathbb{U}} \left(Q(b, u) - \eta_q \sqrt{\frac{\log(N(b))}{N(b, u)}}\right)
			\label{eq:Qominus}
	\end{align}
	where $Q(b, u)$ is as defined in Section~\ref{ssec:pomdp}, and $N(b)$ and $N(b, u)$ are the visitation counts for a belief node and an intermediate belief-action node, respectively.
	$\eta_q$ is a constant for exploration bonus in the tree policy.
	As $\eta_q$ gets larger, PO-UCT becomes more explorative in action selection.
	
	The rollout policy may be a random policy. If there is domain knowledge available, 
	a preferred action set can be specified for the rollout policy to guide the Monte Carlo simulation toward a promising subspace.

	
	After each Monte Carlo simulation, $Q(b, u)$ is updated by the following rule.
	\begin{align}
		Q'(b, u) = Q(b, u) + \frac{R - Q(b, u)}{N(b, u)}
	\end{align}
	where $Q'(b, u)$ denotes the updated $Q(b, u)$ value.
	This allows POMCP to learn the Q-values gradually from consecutive simulations.

	\begin{algorithm}[t]
	{\fontsize{8.5pt}{9.8pt}\selectfont
		\caption{\textsc{OnlinePlanningAndExecution}()}
		\label{alg:firmcp}
		\begin{algorithmic}[1]

			\Input
				\ {
					\par \noindent $\mathcal{b}_0$: an initial belief state
					\par \noindent $B^{goal}$: a goal FIRM node
				}


			\Procedure{OnlinePlanningAndExecution$(\mathcal{b}_0, B^{goal})$}{}

				\State $\mathcal{b} \gets \mathcal{b}_0$

				\While{$\mathcal{b} \notin B^{goal}$}
					\State $u^* \gets  \textsc{Search}(\mathcal{b})$
					\State $z' \gets \textsc{ExecuteAndObserve}(u^*)$
					\State $(\mathcal{b}', c) \gets \textsc{EvolveBeliefState}(\mathcal{b}, u^*, z')$
					\State $\mathcal{b} \gets \mathcal{b}'$
				\EndWhile


			\EndProcedure
		\end{algorithmic}
	}
	\end{algorithm}

	\begin{algorithm}[t]
	{\fontsize{8.5pt}{9.8pt}\selectfont
		\caption{\textsc{Search}()}\label{alg:search}

		\begin{algorithmic}[1]
		
				
			\Procedure{Search$(\mathcal{b})$}{}



				\For{$i = 1, 2, ..., N_{p}$}
					\Comment {$N_{p}$: the number of particles}
					\State $x \sim \mathcal{b}$
						\Comment{draw a sample from belief $\mathcal{b}$}
					\State $\textsc{Simulate}(x, \mathcal{b}, 0, \texttt{nil})$
				\EndFor

				\State $b \gets \textsc{GetMatchingBeliefNode}(\mathbb{T}, \mathcal{b})$
					\Statex\Comment{$\mathbb{T}$: the current POMCP tree}
				\State $\mu^* \gets \argmin_{\mu^{\cdot j} \in \mathbb{M}(b)} Q(b, \mu^{\cdot j}(b))${\color{red} ??mb?}
				\State $u^* \gets \mu^*(b)$
				\State \Return $u^*$

			\EndProcedure
		\end{algorithmic}
	}
	\end{algorithm}

	{\color{red} two parts: pomcp and fast-pomcp}
	\rev{
	In the original POMCP, $J(b)$ in Eq. (\ref{eq:minq}) is not backed up in favor of \textit{unbiasedness}.
	}
	{\color{red} ???}
	This is reasonable in dynamic or adversarial environments, such as Go, which are the target scenarios of the original POMCP \cite{silver2010monte,gelly2011monte}.
	However, it should be noted that the lack of $J(b)$ backup leads to higher variance and slower convergence of $Q(b, u)$. 
	{\color{red} ???}
	
	This can cause a significant problem in RAL-POMDPs. {\color{red} the example/paragraph is not very clear.}
	Firstly, an exceptional{\color{red} ???} simulation episode can easily ruin {\color{red} informal} the learned Q-values in a RAL-POMDP with a longer horizon without discounting {\color{red} ???}.
	For example, consider a scenario where there is a safety risk near to the goal but far from the current state, and the optimal solution path is to dodge this risk once the robot comes closer to it.
	Without discounting, this delayed risk will percolate to all the values along its forward simulation path, 
	and then POMCP without $J(b)$ backup will end up spending a lot of time to explore many other paths branched off near the current state.
	Secondly, the robot simulation and safety evaluation are orders of magnitude more expensive than the game evaluation.
	In the case the environment is mostly static, it is typically more efficient to pursue optimistic action selection rather than unbiased one.
	Thus, it is highly discouraged for complex physical systems to perform many simulations with a slow convergence rate as in the original POMCP.
		
	To remedy this, \ouralgo{} backs up $J(b)$ as described in Eq. (\ref{eq:minq}) in addition to $N(b)$, $N(b, u)$, and $Q(b, u)$.
	As can be seen in line \ref{alg:updateJ} of Algorithm~\ref{alg:simulate}, $J(b)$ is updated at every iteration when $Q(b, u)$ for $\forall u \in \mathbb{U}$ is updated. {\color{red} needs to be highlighted}
	For the more details on online planning, see Algorithm~\ref{alg:firmcp}--\ref{alg:rollout}.
}  

	\subsection{Short-range Local Planner}
	\label{ssec:pomcp}

  The short-range local planner is to find a locally near-optimal policy in an online manner.
  To tackle RAL-POMDP problems, we develop a variant of POMCP (Partially Observable Monte Carlo Planning) \cite{silver2010monte} here, referred to as J-POMCP for the current instance of \ouralgo{}.

  We start by a brief review of the original POMCP algorithm.
	POMCP is an online POMDP solver that uses Monte Carlo Tree Search (MCTS) in belief space and particle representation of belief states.
	POMCP's action selection during Monte Carlo simulation is governed by two policies: a tree policy within the constructed belief tree, and a rollout policy beyond the tree and up to a pre-defined finite discount horizon.
	
  The tree policy selects an action based on Partially Observable UCT (PO-UCT) algorithm as follows.
	\begin{align}
    u^* &= \argmin_{u \in \mathbb{U}} \left(Q(b, u) - \eta_q \sqrt{\frac{\log(N(b))}{N(b, u)}}\right)
			\label{eq:Qominus}
	\end{align}
	where $Q(b, u)$ is as defined in Section~\ref{ssec:pomdp}, and $N(b)$ and $N(b, u)$ are the visitation counts for a belief node and an intermediate belief-action node, respectively.
	$\eta_q$ is a constant for exploration bonus in the tree policy.
	As $\eta_q$ gets larger, PO-UCT becomes more explorative in action selection.
	
	The rollout policy may be a random policy. If there is domain knowledge available, 
	a preferred action set can be specified for the rollout policy to guide the Monte Carlo simulation toward a promising subspace.

	
	After each Monte Carlo simulation, $Q(b_k, u_k)$ (corresponding to the belief-action pair on the $k$-th simulation step in tree) is updated as follows.
	\begin{align}
    Q'(b_k, u_k) &= Q(b_k, u_k) + \frac{R_k - Q(b_k, u_k)}{N(b_k, u_k)}
	\end{align}
  where
	\begin{align}
    R_k &= \sum_{k'=k}^{K^{sr}-1} \gamma^{(k'-k)} c(b_{k'}, u_{k'}) + \gamma^{(K^{sr}-k)} J(b_{K^{sr}})
    \label{eq:return}
	\end{align}
	$Q'(b_k, u_k)$ denotes the updated value of $Q(b_k, u_k)$,
  and $R_k$ is the accumulated return from the horizon of the short-range planner to the current simulation step $k$. 
  Note that $J(b_{K^{sr}})$ should appear in Eq. (\ref{eq:return}) if the $K^{sr}$ is shorter than the problem's planning horizon. 
	Through iterative forward simulations, POMCP gradually learns the Q-value for each belief-action pair.

  There are two major challenges for POMCP when applied to RAL-POMDP problems.
  \begin{enumerate}
    \item \label{enum1} RAL-POMDPs are infinite horizon problems without cost discounting.
      Thus, POMCP needs an estimate of $J(b)$ for each $b$ on its finite horizon, possibly from naive heuristics or sophisticated global policy solvers.
    \item \label{enum2} RAL-POMDPs incorporates computationally expensive costs and constraints, such as collision checking by high-fidelity simulator, and thus, a higher number of forward simulations are discouraged.
      However, POMCP usually requires many simulations until convergence because its update rule in Eq. (\ref{eq:return}) does not bootstrap all the successors whose values are initialized by domain knowledge or global policy solvers.
  \end{enumerate}

  The first challenge is addressed by the bridge planner that connects the short-range planner to the long-range global planner (see Section \ref{sec:bridge}). To handle the second challenge, we develop a variant of POMCP, referred to as J-POMCP.

	\begin{algorithm}[t]
	{\fontsize{8.5pt}{9.8pt}\selectfont
		\caption{\textsc{OnlinePlanningAndExecution}()}
		\label{alg:firmcp}
		\begin{algorithmic}[1]

			\Input
				\ {
					\par \noindent $\mathcal{b}_0$: an initial belief state
					\par \noindent $B^{goal}$: a goal FIRM node
				}


			\Procedure{OnlinePlanningAndExecution$(\mathcal{b}_0, B^{goal})$}{}

				\State $\mathcal{b} \gets \mathcal{b}_0$

				\While{$\mathcal{b} \notin B^{goal}$}
					\State $u^* \gets  \textsc{Search}(\mathcal{b})$
					\State $z' \gets \textsc{ExecuteAndObserve}(u^*)$
					\State $(\mathcal{b}', c) \gets \textsc{EvolveBeliefState}(\mathcal{b}, u^*, z')$
					\State $\mathcal{b} \gets \mathcal{b}'$
				\EndWhile


			\EndProcedure
		\end{algorithmic}
	}
	\end{algorithm}

	\begin{algorithm}[t]
	{\fontsize{8.5pt}{9.8pt}\selectfont
		\caption{\textsc{Search}()}\label{alg:search}

		\begin{algorithmic}[1]
		
				
			\Procedure{Search$(\mathcal{b})$}{}



				\For{$i = 1, 2, ..., N_{p}$}
					\Comment {$N_{p}$: the number of particles}
					\State $x \sim \mathcal{b}$
						\Comment{draw a sample from belief $\mathcal{b}$}
					\State $\textsc{Simulate}(x, \mathcal{b}, 0, \texttt{nil})$
				\EndFor

				\State $b \gets \textsc{GetMatchingBeliefNode}(\mathbb{T}, \mathcal{b})$
					\Statex\Comment{$\mathbb{T}$: the current POMCP tree}
				\State $\mu^* \gets \argmin_{\mu^{\cdot j} \in \mathbb{M}(b)} Q(b, \mu^{\cdot j}(b))$
          \revv{\Statex\Comment{$\mathbb{M}(b)$: the set of local controllers applicable to $b$}}
				\State $u^* \gets \mu^*(b)$
				\State \Return $u^*$

			\EndProcedure
		\end{algorithmic}
	}
	\end{algorithm}


  J-POMCP follows exactly how POMCP selects actions to explore the belief space, but slightly differs in how the values are updated.
  More precisely, J-POMCP uses the following instead of Eq. (\ref{eq:return}) to compute $Q'(b_k, u_k)$.
	\begin{align}
    R_k &= c(b_k, u_k) +  \min_{u \in \mathbb{U}} Q(b_{k+1}, u)
	\end{align}
  Note that the second term on the right-hand side is $J(b_{k+1})$ in Eq. (\ref{eq:minq}), hence the name J-POMCP.
  This is in fact how Q-learning updates the Q-value in a greedy manner by bootstrapping the initialized or learned values \cite{watkins1992q}.
  It creates a bias in value and converges faster if the initial values are informative.

  In \ouralgo{}, we have access to the approximate global policy computed by the long-range solver, which is much better than the heuristics computed under the assumption of deterministic or fully observable environments.
  J-POMCP can make the most use of the underlying global policy through bootstrapping.

  This new update rule can be implemented as presented in Algorithm~\ref{alg:simulate} (see Line~\ref{alg:updateR}, \ref{alg:updateJ}, and \ref{alg:returnJ}).
	As in Line \ref{alg:updateJ}, $J(b)$ is updated every time $Q(b, u)$ is updated for $\forall u \in \mathbb{U}$. 

	Algorithms~\ref{alg:firmcp}--\ref{alg:rollout} detail this online planning process.

	\begin{algorithm}[t]
	{\fontsize{8.5pt}{9.8pt}\selectfont
		\caption{\textsc{Simulate()}}\label{alg:simulate}

		\begin{algorithmic}[1]


			\Procedure{Simulate$(x, \mathcal{b}, k, \mu^\smallminus)$}{}

				\If{$k > K^{sr}$}
					\Comment{$K^{sr}$: short-range planner's fixed horizon}
					\State \Return $\textsc{Rollout}(x, \mathcal{b}, k, \mu^\smallminus)$
				\EndIf

				\State $b \gets \textsc{GetMatchingBeliefNode}(\mathbb{T}, \mathcal{b})$

				\If {$b$ is \texttt{nil}}

					\State $\mathbb{T} \gets \textsc{AddNewBeliefNodeToTree}(\mathbb{T},\mathcal{b})$
					\State $b \gets \mathcal{b}$
					\State $\mathbb{M}(b) \gets \{\}$
						\Comment{$\mathbb{M}(b)$: a set of local controllers for $b \in \mathbb{T}$}
					\State $\mathbb{N}(b) \gets \textsc{GetNearestNeighbors}(b)$
					\ForAll {$j$ s.t. $B^j \in \mathbb{N}(b)$}
						\State $\mu^{\cdot j} \gets \textsc{GetLocalController}(b, B^{j^*})$
						\State $\mathbb{M}(b) \gets \mathbb{M}(b) \cup \{\mu^{\cdot j}\}$
						\State $\mathscr{C}^j \gets \textsc{HeuristicEdgeCost}(b, B^j)$
						\State $Q(b, \mu^{\cdot j}(b)) \gets \mathscr{C}^j + \tilde{J}^g(B^j)$
							\label{alg:initQ}
						\State $N(b, \mu^{\cdot j}(b)) \gets 0$
							\Statex \Comment{$N(b, u)$: visitation count of  $b \in \mathbb{T}$ and $u \in \mathbb{U}$}
					\EndFor
					\State $J(b) \gets \min_{\mu^{\cdot j} \in \mathbb{M}(b)} Q(b, \mu^{\cdot j}(b))$
							\label{alg:initJ}
					\State $N(b) \gets 0$
						\Comment{$N(b)$: visitation count of  $b \in \mathbb{T}$}

					\State \Return $\textsc{Rollout}(x, \mathcal{b}, k, \mu^\smallminus)$

				\Else

					\State $b \gets \textsc{UpdateBeliefNode}(b, \mathcal{b})$



					\State $\mu^* \gets \argmin_{\mu^{\cdot j} \in \mathbb{M}(b)} Q(b, \mu^{\cdot j}(b)) - \eta_q \sqrt{\frac{\log N(b)}{N(b, \; \mu^{\cdot j}(b))}}$
						\Statex \Comment{$\eta_q$: tree exploration parameter}

					\State $u^* \gets \mu^*(b)$

					\State $(x', z') \gets \textsc{GenerativeModel}(x, u^*)$
					\State $(\mathcal{b}', c) \gets \textsc{EvolveBeliefState}(\mathcal{b}, u^*, z')$

					\State $R \gets c' + \textsc{Simulate}(x', \mathcal{b}', k\!+\!1, \mu, j^\smallplus)$
					    \label{alg:updateR}
          \label{alg:return}

				\EndIf

				\State $N(b, u^*) \gets N(b, u^*) + 1$
				\State $N(b) \gets N(b) + 1$

				\State $Q(b, u^*) \gets Q(b, u^*) + \frac{R - Q(b, \; u^*)}{N(b, \; u^*)}$

				\State $J(b) \gets \min_{\mu^{\cdot j} \in \mathbb{M}(b)} Q(b, \mu^{\cdot j}(b))$
				    \label{alg:updateJ}


				\State \Return $J(b)$
				    \label{alg:returnJ}

			\EndProcedure
		\end{algorithmic}
	}
	\end{algorithm}

\vspace{-10pt}

	\begin{algorithm}[t]
	{\fontsize{8.5pt}{9.8pt}\selectfont
		\caption{\textsc{Rollout()}}\label{alg:rollout}

		\begin{algorithmic}[1]


			\Procedure{Rollout$(x, \mathcal{b}, k, \mu^\smallminus)$}{}

				\If{$k > K^{sr}$}


					\State $B^{j^\smallminus} \gets \textsc{GetTargetFIRMNode}(\mu^\smallminus)$
					\If{$\mathcal{b} \in B^{j^\smallminus}$}
						\State \Return $\tilde{J}^g(B^{j^\smallminus})$

					\Else

						\State $\mu^* \gets \mu^\smallminus$
						\State $u^* \gets \mu^*(b)$

					\EndIf

				\Else

					\State $\mathbb{N}(\mathcal{b}) \gets \textsc{GetNearestNeighbors}(\mathcal{b})$

					\ForAll {$j$ s.t. $B^j \in \mathbb{N}({\mathcal{b}})$} 
						\State $\mathscr{C}^j \gets \textsc{HeuristicEdgeCost}(\mathcal{b}, B^j)$
						\State $w^{j} \gets \frac{1}{\mathscr{C}^j + \tilde{J}^g(B^j)} + \eta_w$
							\label{alg:rollout-Jreg}
							\Statex \Comment{$\eta_w$: rollout exploration parameter}
					\EndFor
					\State $j^* \sim \frac{w^j}{\sum_{j'} w^{j'}}$

					\State $\mu^* \gets \textsc{GetLocalController}(\mathcal{b}, B^{j^*})$
					\State $u^* \gets \mu^*(\mathcal{b})$
				\EndIf

				\State $(x', z') \gets \textsc{GenerativeModel}(x, u^*)$
				\State $(\mathcal{b}', c) \gets \textsc{EvolveBeliefState}(\mathcal{b}, u^*, z')$

				\State \Return $c + \textsc{Rollout}(x', \mathcal{b}', k\!+\!1, \mu^*)$

			\EndProcedure

		\end{algorithmic}
	}
	\end{algorithm}

	\subsection{Bridging the Local and Global Policies}\label{sec:bridge}



	The major shortcomings of the traditional short range planners (e.g., POMCP) in RAL-POMDP problems are due to the lack of proper guidance beyond its horizon.
	These methods may fall into the local minima and lead to a highly risky or suboptimal solution.
	In contrast, \ouralgo{} bootstraps the graph-based global policy to guide the forward search during online planning and improve the safety guarantees and optimality.

	There are two places where the cost-to-go information from graph-based global policy (e.g., FIRM) is being used.
	1) When a new belief node $b$ is added to the \ouralgo{}'s forward search tree, it is initialized using the cost-to-go from the underlying global graph (line~\ref{alg:initQ} and \ref{alg:initJ} in Algorithm~\ref{alg:simulate}).
	\begin{align}
		Q_{init}(b, u^j) &= \mathscr{C}(b, B^j) + \tilde{J}^g(B^j) \\
		J_{init}(b) &= \min_{u^j \in \mathbb{U}(b)} Q_{init}(b, u^j)
	\end{align}
	where $\mathbb{U}(b) = \{ \mu^{\cdot j}(b) \, | \, \mu^{\cdot j} \! \in \! \mathbb{M}(b)\}$
	and $\mathbb{M}(b)$ is a set of local controllers that steers the belief from $b$ to its neighboring FIRM nodes $B^j \in \mathbb{N}(b)$.
	$\mathscr{C}(b, B^j)$ is the estimated edge cost from $b$ to $B^j$, and $\tilde{J}^g(B^j)$ is the cost-to-go computed by FIRM in the offline phase.
	The visitation counts are initialized to zeros, i.e., $N_{init}(b, u^j) = 0$ and $N_{init}(b) = 0$.
	Note that the action space $\mathbb{U}(b)$ is only a subset of the continuous action space which is based on local controllers toward neighboring FIRM nodes.
		%
	%
	
	2) The rollout policy selects an action by random sampling from a probability mass function which is based on FIRM's cost-to-go rather than a uniform distribution (Algorithm~\ref{alg:rollout}).
	For each $B^j \in \mathbb{N}(\mathcal{b})$, the weight $w^j$ is computed as
	\begin{align}
		w^j &= (\mathscr{C}(\mathcal{b}, B^j) + \tilde{J}^g(B^j))^{-1} + \eta_w
	\end{align}
	where $\mathcal{b}$ denotes the sampled belief state in the current Monte Carlo simulation. $\eta_w$ is a constant for the exploration bonus in the rollout policy. As $\eta_w$ gets larger, the rollout policy becomes more explorative. $\eta_w = \infty$ will lead the rollout to pure exploration, i.e., random sampling from uniform distribution.
	
	Based on the computed weight $w^j$, the rollout policy selects an action by random sampling $u^* \sim p(u^j; \mathcal{b})$ from the following probability mass function.
	\begin{align}
		p(u^j; \mathcal{b}) &= {w^j}{(\textstyle \sum_{j'} w^{j'})^{-1}}~\text{ for } \forall u^j \in \mathbb{U}(\mathcal{b})
	\end{align}

\vspace{-8pt}

\rev{
	\subsection{Discussion}

	We briefly highlight a few properties of the proposed algorithm in terms of optimality and safety.

	\pr{Optimality}
	We first consider a small problem where the goal belief (with the known cost-to-go of 0) is within the finite horizon of POMCP from the beginning.
	Stand-alone POMCP is guaranteed to converge to the optimal solution \cite{silver2010monte}, and it is trivial to prove that \ouralgo{} converges to the optimal.
	In larger problems, the global optimality depends on the cost-to-go estimation for the leaf nodes of the POMCP tree.
	Note that finding the accurate cost-to-go estimation is as difficult as the original problem. 
	Simple heuristics such as Euclidean distance heuristic typically  provide poor cost-to-go estimation compared to the approximate long-range solvers, such as FIRM, that take uncertainty into account.
	Hence, in terms of optimality, \ouralgo{} mostly outperforms POMCP.
%
	It should also be noted that the cost-to-go estimation of FIRM gets closer to the optimal with more and more samples \cite{Ali14-IJRR},
	which can improve the overall optimality of \ouralgo{}.

	\pr{Safety}
	As discussed in Sec.~\ref{ssec:ral-pomdp}, minimizing the risk (the expected failure probability along the whole trajectory) is encoded as a soft constraint in the cost-to-go minimization problem.
	Based on the optimality analysis, \ouralgo{} can achieve smaller expected cost-to-go from the initial belief than POMCP in non-trivial problems, which effectively leads to policies with less risk than POMCP policies.
	Since \ouralgo{} adapts to the current belief during the online planning phase, it can provide higher safety than (offline) FIRM planner.

}  

\section{Experiments}  \label{sec:experiments}


\subsection{Rover Navigation Problem}
As a representative RAL-POMDP problem, we consider the real-world problem of the Mars rover navigation under motion and sensing uncertainty.
In Rover Navigation Problem (RNP) introduced here, the objective is to navigate a Mars rover from a starting point to a goal location while avoiding risk regions such as steep slopes, large rocks, etc.
The rover is provided a map of the environment which is created by a Mars orbiter satellite \cite{HiRiseMapLocalization} and a Mars helicopter \cite{MarsHeli} flying ahead of the rover.
This global map contains the location of landmarks, which serve as information sources that the rover can use to localize itself on the global map. 
The map also contains the location of risk regions that the rover needs to avoid and regions of science targets which needs to be visited by the rover to collect data or samples.


\ifarxivfin

\pr{Motion model}
The motion of the rover is noisy due to factors like wheel slippage, unknown terrain parameters, etc.
In RNP introduced here, we assume a nonlinear motion model (but still a holonomic one to provide a simple benchmark). Specifically, we use the model in \cite{Tamas04}, where the state $x = [\prescript{g}{}{x},\prescript{g}{}{y},\prescript{g}{}{\theta}]^T \in \mathbb{R}^3$ represents the 2D position and heading angle of the rover in the global world frame. Control input $u \in \mathbb{R}^3$ represents the velocity of each coordinate. Using \cite{Tamas04}, we obtain the discrete motion model as follows:
\begin{align}
    x_k = f(x_{k-1},u_{k-1},w_{k-1})    
\end{align}
where $w$ is motion noise drawn from a Gaussian distribution with zero mean.

\newcommand\inv[1]{#1\raisebox{1.15ex}{$\scriptscriptstyle-\!1$}}
\pr{Observation model}
In RNP, we assume the rover can measure the range and bearing to each landmark.
Denoting the displacement vector to a landmark $L^i$ by $d^i = [d^i_x, d^i_y]^T \equiv L^i - p$, where $p=[\prescript{g}{}x, \prescript{g}{}y]^T$ is the position of the robot, the observation model is given by:
\begin{align}
	z^i &= h^i(x,v^i) = \left[ ||d^i||,\; \inv{\tan\!\!}(d^i_y/d^i_x)\! -\! \prescript{g}{}\theta \right]^T \!+ v^i \\ 
    R^i &= \diag \left((\xi_r||d^i||+\sigma^r_b)^2,\; (\xi_\theta||d^i||+\sigma_b^\theta)^2 \right)
\end{align}
where $v^i \sim \mathcalorg{N}(0, R^i)$.
The measurement quality degrades as the distance of the robot from the landmark increases, and the weights $\xi_r$ and $\xi_\theta$ control this dependency. 
$\sigma^r_b$ and $\sigma^\theta_b$ are the bias standard deviations.

\fi  

\pr{Cost and risk metrics}
In RNP, we consider the localization accuracy as well as the mission completion time as the main elements of the cost function.
\rev{
Specifically, we consider a cost function under the Gaussian assumption as follows.
\begin{align}
c(b_k,u_k) &= \xi_p \textup{tr}(P_k)  + \xi_T \Delta t
	\label{eq:unitcost}
\end{align}
where $P_k=\textup{cov}(x_k|z_{0:k})$ represents the second moment of the belief distribution as a measure of state uncertainty, and $\Delta t$ is the time step size for each action.
$\xi_p$ and $\xi_T$ denote weights to combine these different objectives.
In the experiments, we used $\xi_p = 10$, $\xi_T = 1$, and $\Delta t = 0.005$.
%
	The risk in RNP denotes the expected probability of failure (collision with obstacles) along the whole trajectory under a policy as described in Section~\ref{ssec:ral-pomdp}.
	Note that the action cost in Eq. (\ref{eq:unitcost}) is not directly related to the risk metric.
}

\pr{RNP scalability}
To test algorithms under different RNP complexities, we parameterize the Rover Navigation Problem as $\rnp{}_{s}(e, o)$, where $e$ represents the size of the environment, $o$ represents the size/density of obstacles, and $s$ is the environment type.
We compare three key attributes (safety, scalability, and optimality) of each algorithm in three different environments ($\rnp{}_{\infotrap{}}$, $\rnp{}_{\obswall{}}$, and $\rnp{}_{\forest{}}$).
%


\subsection{Baseline Methods}

As baseline methods, we consider three algorithms. From the class of forward search methods, we consider the POMCP method \cite{silver2010monte}. From the approximate long-range methods, we consider the FIRM \cite{Ali14-IJRR} and its variant \cite{Ali14-RolloutFIRM-ICRA,agha2018slap} which is referred to as online graph-based rollout (OGR) here.

\pr{FIRM}
FIRM is an execution of closed-loop controls returned by its offline planning algorithm.
FIRM relies on belief-stabilizing local controllers at each graph node to ameliorate the curse of history.
Hence, it can solve larger problems, but it is usually suboptimal compared to optimal online planners.

\pr{OGR}
OGR is an online POMDP solver that improves the optimality of a base graph-based method (particularly, FIRM in our implementation). At every iteration, an OGR planner selects the next action by simulating all different possible actions and picking the best one. 
Compared to \ouralgo{}, OGR expands the belief tree in a full-width but only for one-step look-ahead.
While it can improve the performance of its base graph-based planner, it is prone to local minima due to the suboptimality in base planner's cost-to-go and OGR's myopic greedy policy. 
Additionally, OGR discards the performed forward simulation results in the next iteration, while \ouralgo{} leverages them at each step to enrich the underlying tree structure.

\pr{URM-POMCP} 
We extend POMCP to make it work in larger and continuous spaces such as RAL-POMDPs. We refer to it as URM-POMCP (Uniform RoadMap POMCP).
%
%
\rev{
URM-POMCP uses a heuristic cost-to-go function to cope with the finite horizon limitation in POMCP.
\orgg{
In this work, the heuristic cost-to-go function $\tilde{J}(b, x)$ is implemented as $\tilde{J}(b, x) \approx \frac{d(b, b_g)}{\Delta \dot{x}_{max}} (\xi_p\tr({P_c}) + \xi_T \Delta t)$, where $d(b, b_g)$ is the Euclidean distance from the current belief state to the goal belief state, $\Delta \dot{x}_{max}$ is the (approximate) maximum velocity of the rover, and $P_c$ is the stationary covariance of $m \in \mathbb{X}$ for the current belief $b=(m,P)$.
}  
}  
%



\subsection{Safety}

Reducing risk and ensuring system's safety is the most important goal of the proposed framework. 
We compare the risk aversion capability of BVL with the baseline methods on $\rnp{}_{\infotrap{}}(e,o)$ shown in Fig.~\ref{fig:infotrap}, where $e$ is the length of the environment and $o$ is the length of the obstacle.

In $\rnp{}_{\infotrap{}}$ problems the rover needs to reach the goal by passing through the narrow passage without colliding with any obstacles.
%
%
As shown in Fig.~\ref{fig:infotrap}, \ouralgo{} reduces risk of collision by executing a longer trajectory that goes close to the landmarks (yellow diamonds) and reduces the localization uncertainty before entering the narrow passage. 
Since the URM-POMCP algorithm plans in a shorter horizon and 
depends on a heuristic cost-to-go estimation beyond the horizon, it takes a greedy approach to go towards the goal thus taking a higher risk of colliding with the obstacles. 
This can also be seen in Fig.~\ref{fig:safety_plot} that shows the probability of collision of the rover as the length of the obstacle increases.
The probability of collision here was estimated by running 20 Monte Carlo simulations of rover executing policies by different planners.

\ifarxiv
\else  

\begin{figure}[t]
	\centering
	\begin{subfigure}{0.493\columnwidth}
		\centering
		\includegraphics[clip,trim=0cm 0cm 0cm 0cm, width=\textwidth]{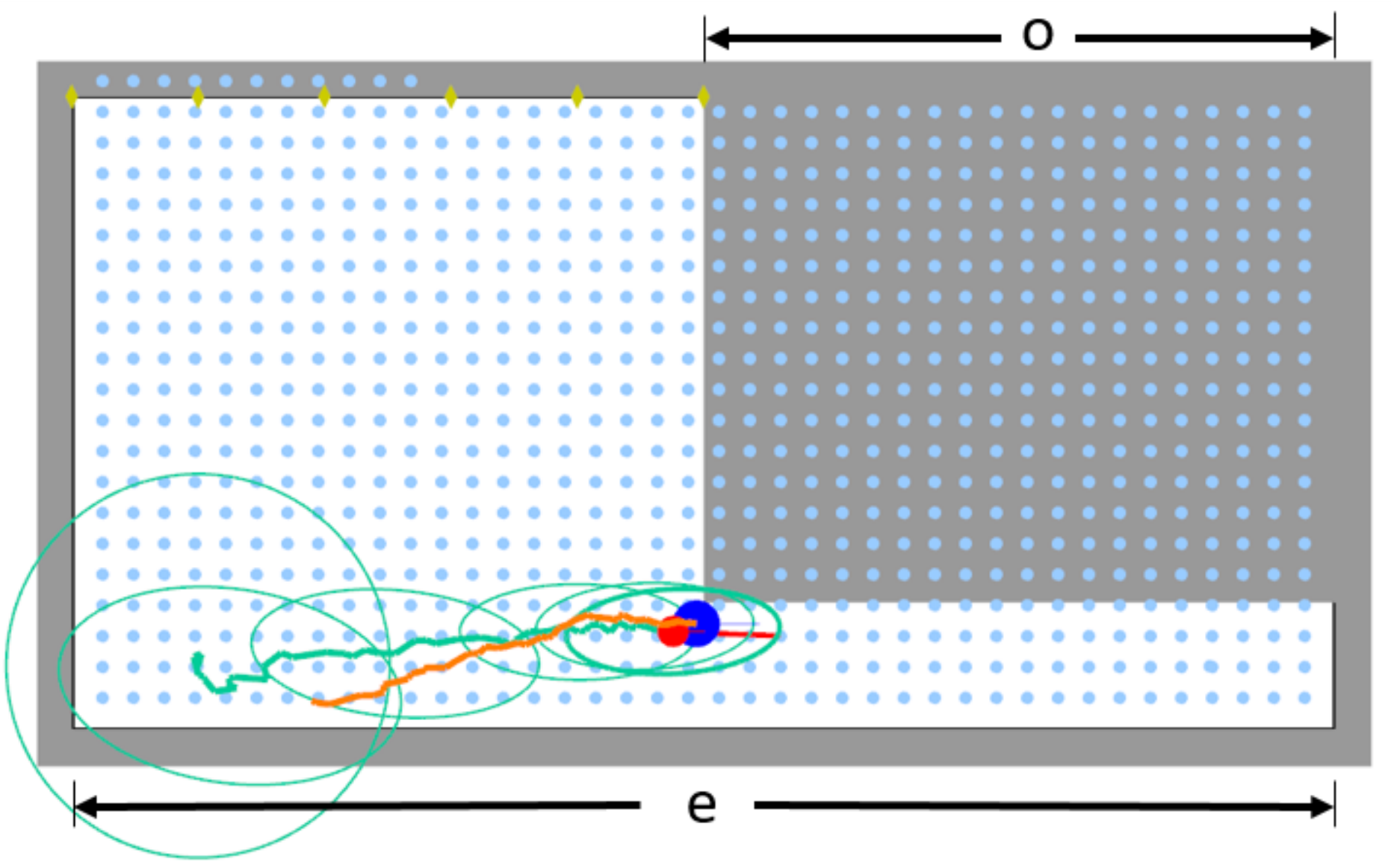}
		\vspace{-0.6cm}
		\caption{POMCP}
	\end{subfigure}
	\begin{subfigure}{0.493\columnwidth}
		\centering
		\includegraphics[clip,trim=0cm 0cm 0cm 0cm, width=\textwidth]{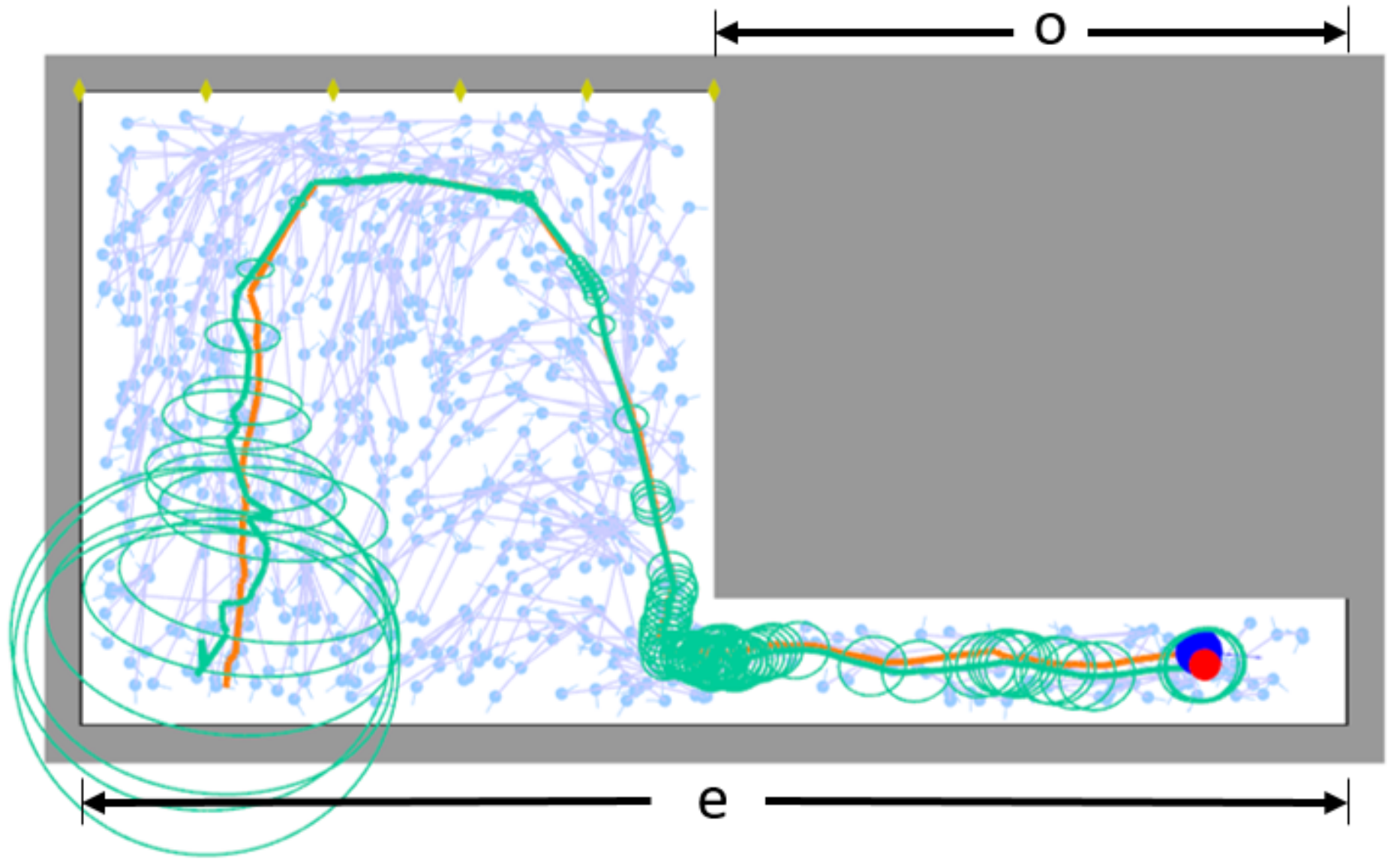}
		\vspace{-0.6cm}
		\caption{\ouralgo{}}
	\end{subfigure}
	\caption{
	\rev{Execution trajectories for $\rnp{}_{\infotrap{}}(10,3)$ problem where Mars rover navigates from start to goal through the narrow passage while avoiding obstacles. The rover can reduce its pose uncertainty by moving closer to the landmarks (yellow diamonds) on top of the map. 
		The \ouralgo{} trajectory approaches the landmarks first and then enters the narrow passage to reduce the chance of collision,
		but the URM-POMCP trajectory aggressively moves toward the goal without considering the risk beyond its finite horizon, which leads to a higher chance of collision. 
    \revv{
    The orange path represents the actual trajectory of the robot, and the green path and ellipses illustrate the mean and covariance of the belief, respectively.
    The underlying tree in light blue depicts the global policy computed by FIRM.
  }
	}  
	}
	\label{fig:infotrap}
\end{figure}

\fi  

\begin{figure}[t]
    \centering
    \includegraphics[height=4cm]{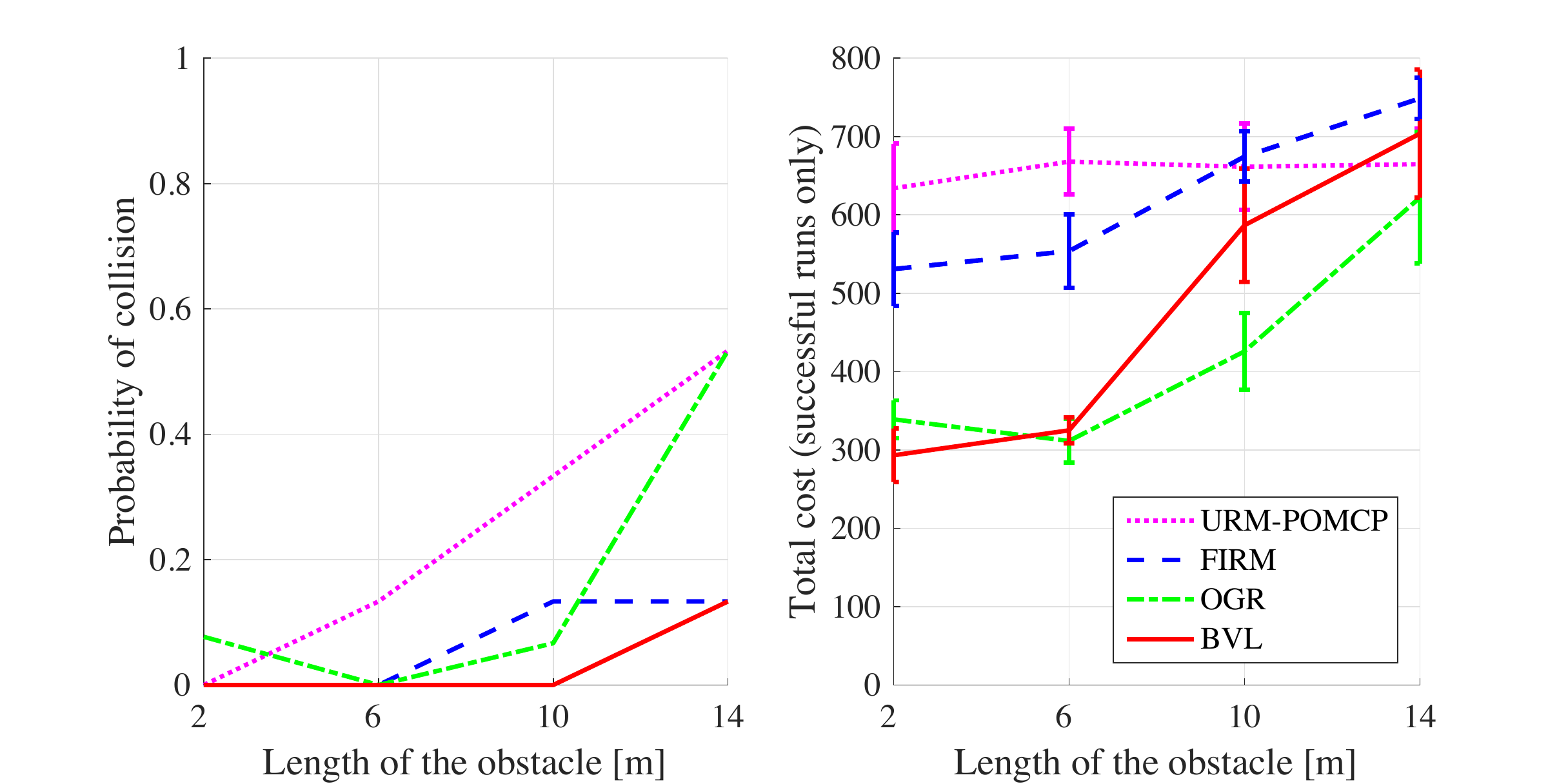}
    \caption{Plots of the probability of collision and the total cost over different obstacle lengths evaluated by running 20 Monte Carlo simulations of rover executing policies generated by the planners. Mean and standard deviation are computed using the successful (collision-free) executions only, which means that algorithms with higher collision probability will have much higher expected total cost when considering the collision penalty.}
    \label{fig:safety_plot}
\end{figure}

\revv{
In this work, the heuristic cost-to-go function $\tilde{J}(b, x)$ is implemented as $\tilde{J}(b, x) \approx \frac{d(b, b_g)}{\Delta \dot{x}_{max}} (\xi_p\tr({P_c}) + \xi_T \Delta t)$, where $d(b, b_g)$ is the Euclidean distance from the current belief state to the goal belief state, $\Delta \dot{x}_{max}$ is the (approximate) maximum velocity of the rover, and $P_c$ is the stationary covariance of $m \in \mathbb{X}$ for the current belief $b=(m,P)$.
This heuristic optimistically assumes that the belief can reach the goal by following the direct path at the maximum velocity without collision with the obstacles.
}  


To construct a finite action set for URM-POMCP, we utilize a uniformly distributed roadmap in belief space. Each point in the uniformly distributed roadmap serves as the target point of a time-varying LQG controller, so that the controller can generate control inputs for a belief to move toward the point.
This enables POMCP to utilize the Gaussian belief model in generating control inputs and updating the belief from observations.
\rev{
Additionally, we penalize the actions to stay at the same state to prevent the robot from getting stuck at local minima indefinitely.
}

\subsection{Scalability in Planning Horizon}
Bi-directional learning of the value function enables the proposed planner to scale to infinite-horizon planning problems with terminal state.
To compare this scalability with the baseline methods, we consider $\rnp{}_{\obswall{}}(e,o)$ problems shown in Fig.~\ref{fig:wall}, where $e$ is the length of the environment and $o$ is the length of the obstacle shown in the figure.

Notice that in Fig.~\ref{fig:scalability}, as the obstacle gets larger, the local minimum gets deeper and the performance of URM-POMCP becomes worse.
The number of time steps to get to the goal grows exponentially for URM-POMCP, while it grows linearly for \ouralgo{} and others.
This shows the effectiveness of guidance by long-range solver's global policy in larger problems as opposed to a naive heuristic guidance in URM-POMCP.

\begin{figure}[t]
	\centering
	\begin{subfigure}{0.4\columnwidth}
		\centering
		\includegraphics[width=\textwidth]{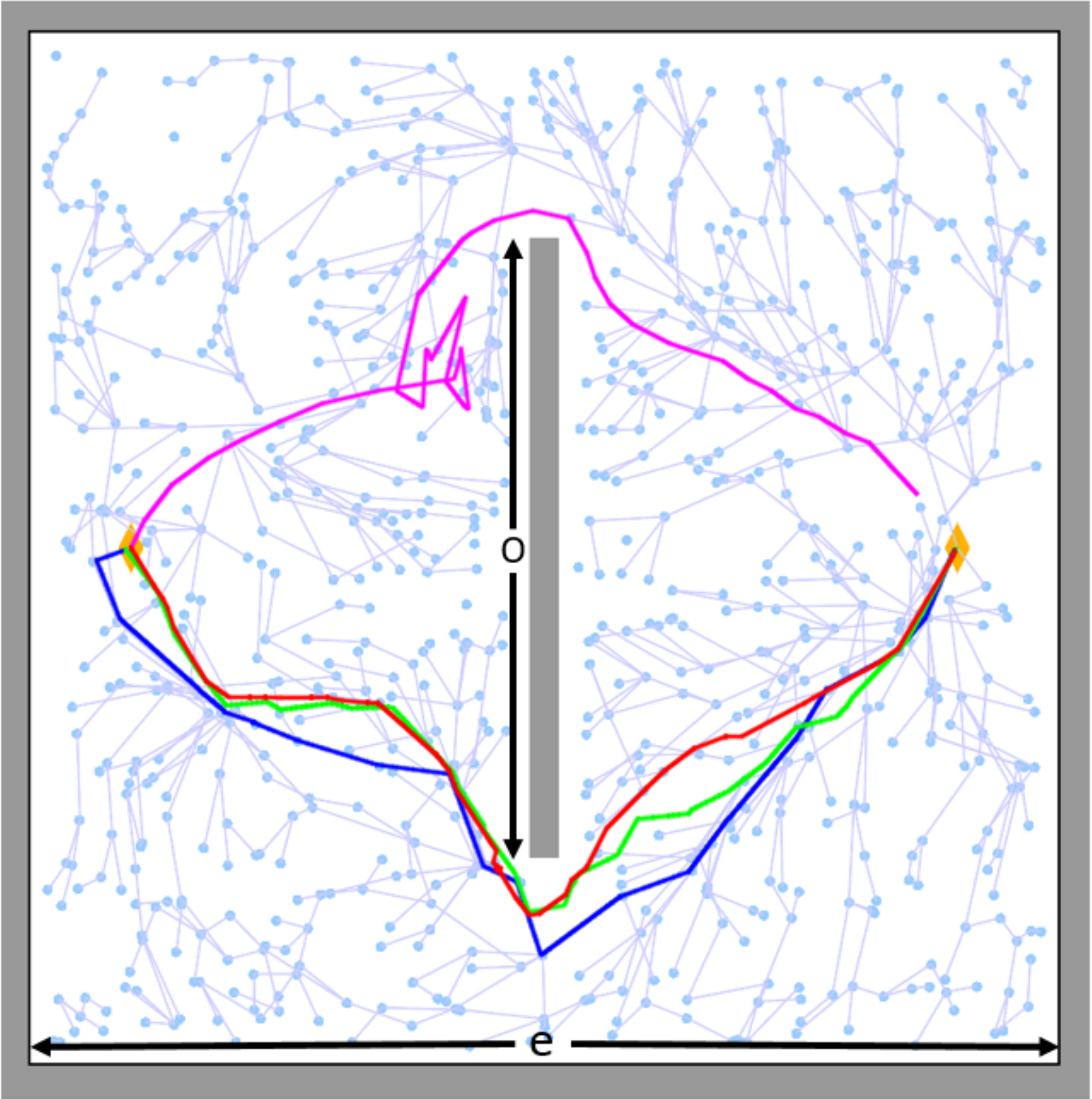}
		\caption{$\rnp{}_{\obswall{}}(20,10)$}
	\end{subfigure}
	\,
	\begin{subfigure}{0.4\columnwidth}
		\centering
		\includegraphics[width=\textwidth]{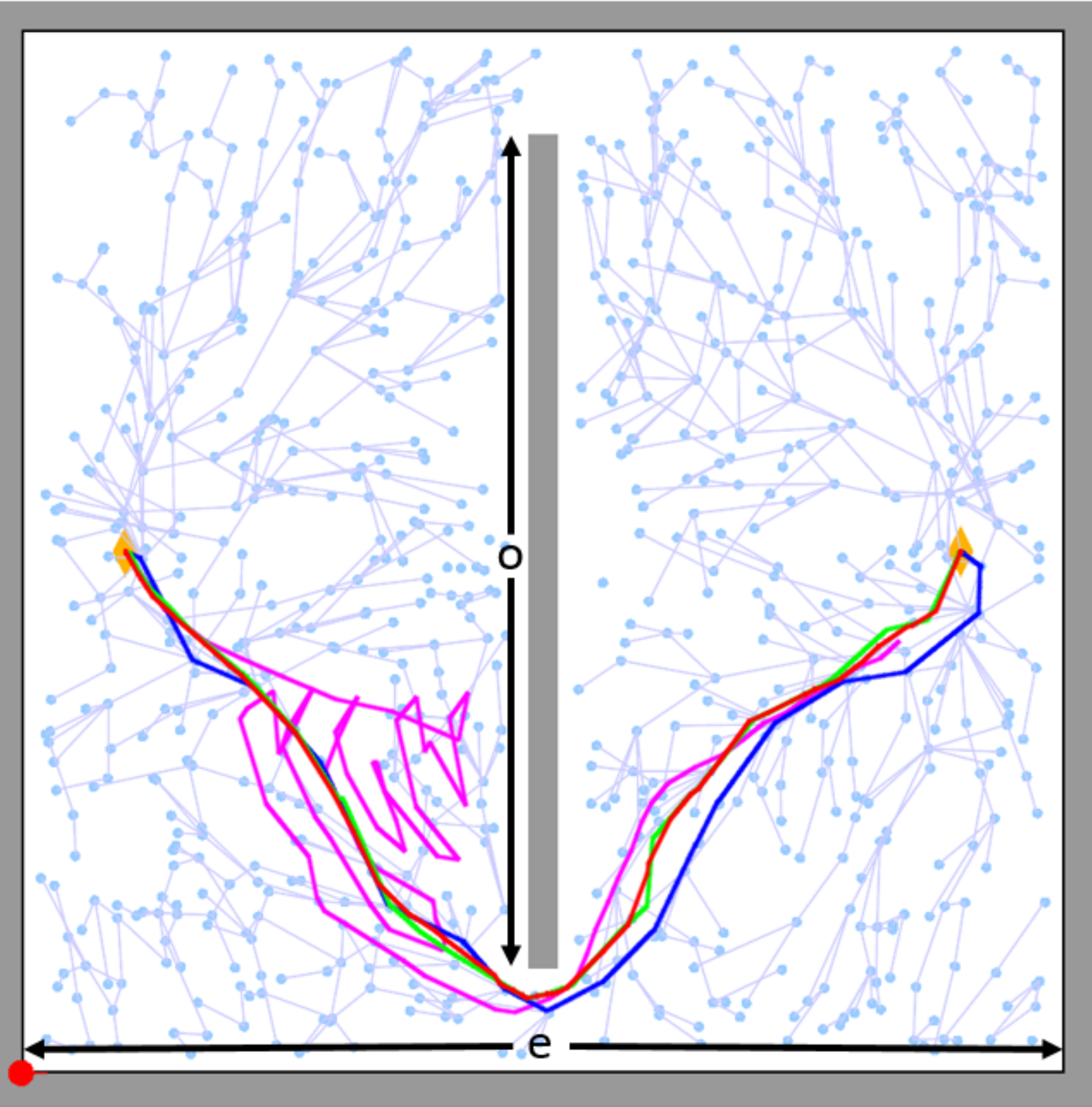}
		\caption{$\rnp{}_{\obswall{}}(20,16)$}
	\end{subfigure}
	\caption{Execution trajectories of URM-POMCP (pink), FIRM (blue), OGR (green), and \ouralgo{} (red).
    The start and the goal states are on the left and the right of the wall, respectively,
    \revv{and the underlying tree in light blue depicts the global policy computed by FIRM.}
  }
	\label{fig:wall}
\end{figure}

\begin{figure}[t]
        \begin{minipage}{0.42\columnwidth}
		\includegraphics[height=4cm]{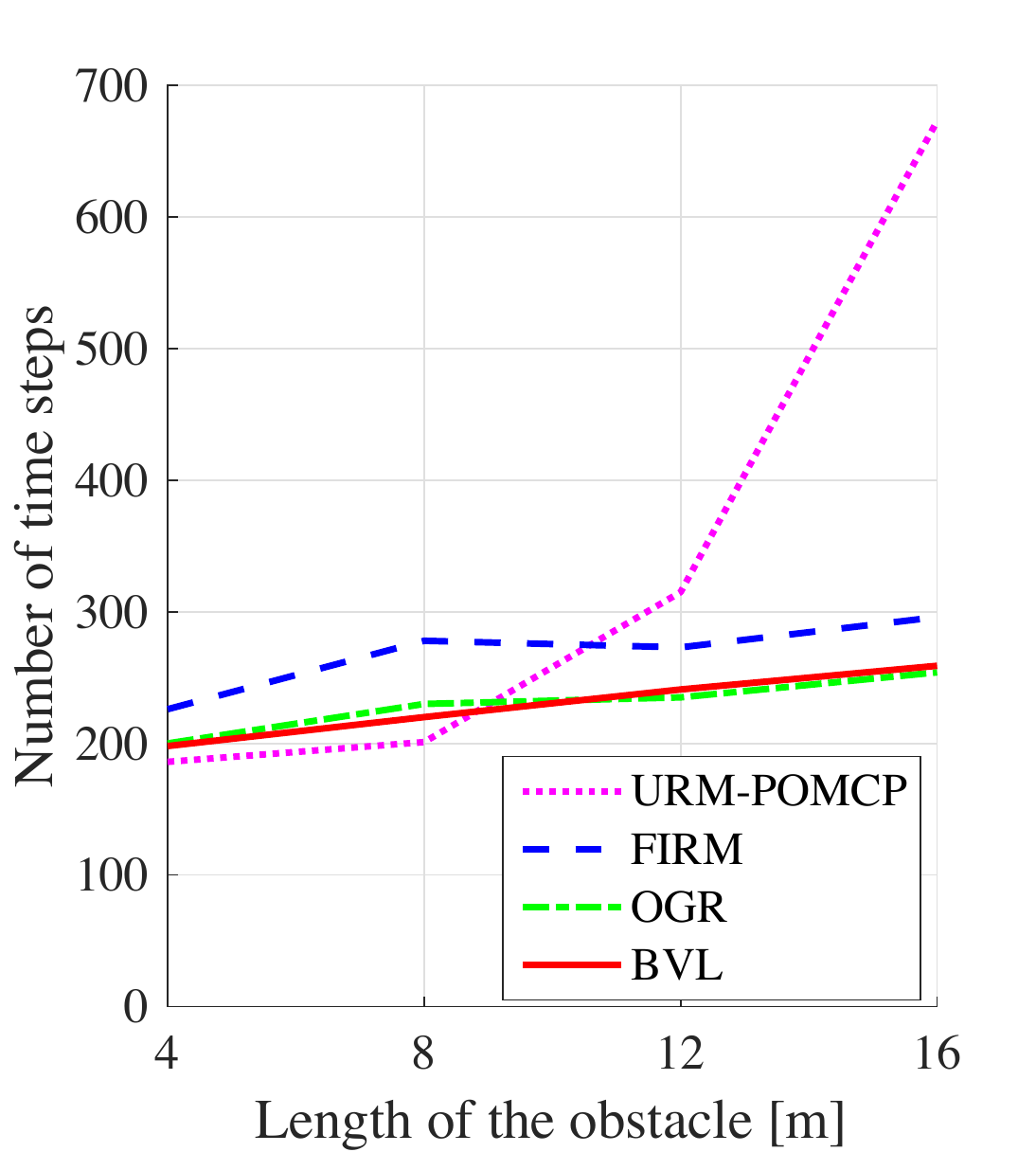}
		\end{minipage}
		\begin{minipage}{0.58\columnwidth}
	\caption{Plot for scalability tests. A longer wall induces deeper local minima for URM-POMCP due to its heuristic cost-to-go estimation.
	URM-POMCP performs worse as having more local minima, but other methods with FIRM's approximate cost-to-go are less affected.}
		\vspace{-0.5cm}
	\label{fig:scalability} 
    \end{minipage}
\end{figure}
\subsection{Optimality}

\begin{figure}[t]
	\centering
	\begin{subfigure}{0.4\columnwidth}
		\centering
		\includegraphics[width=\textwidth]{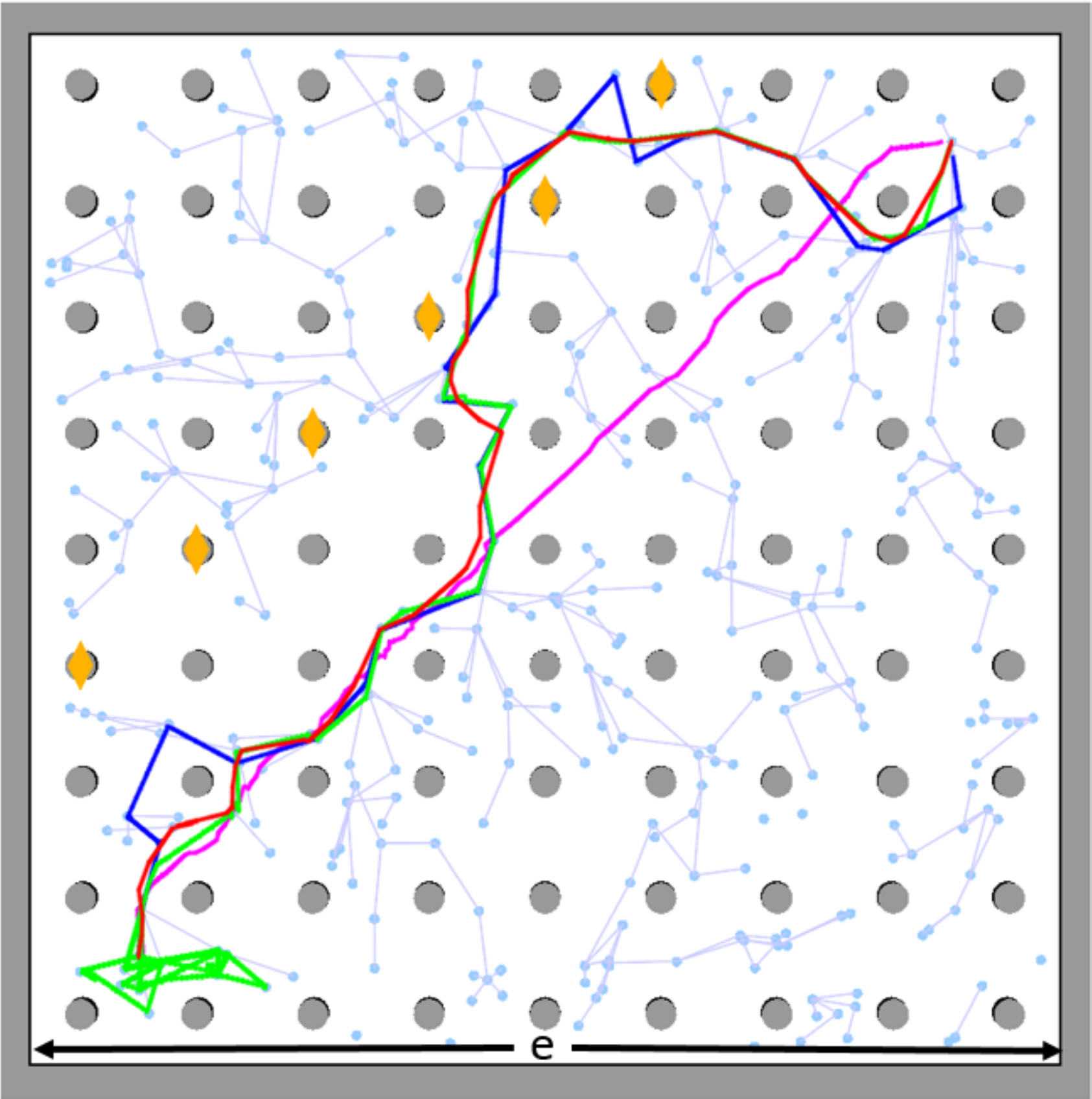}
		\caption{With 350 FIRM nodes}
	\end{subfigure}
	\,
	\begin{subfigure}{0.4\columnwidth}
		\centering
		\includegraphics[width=\textwidth]{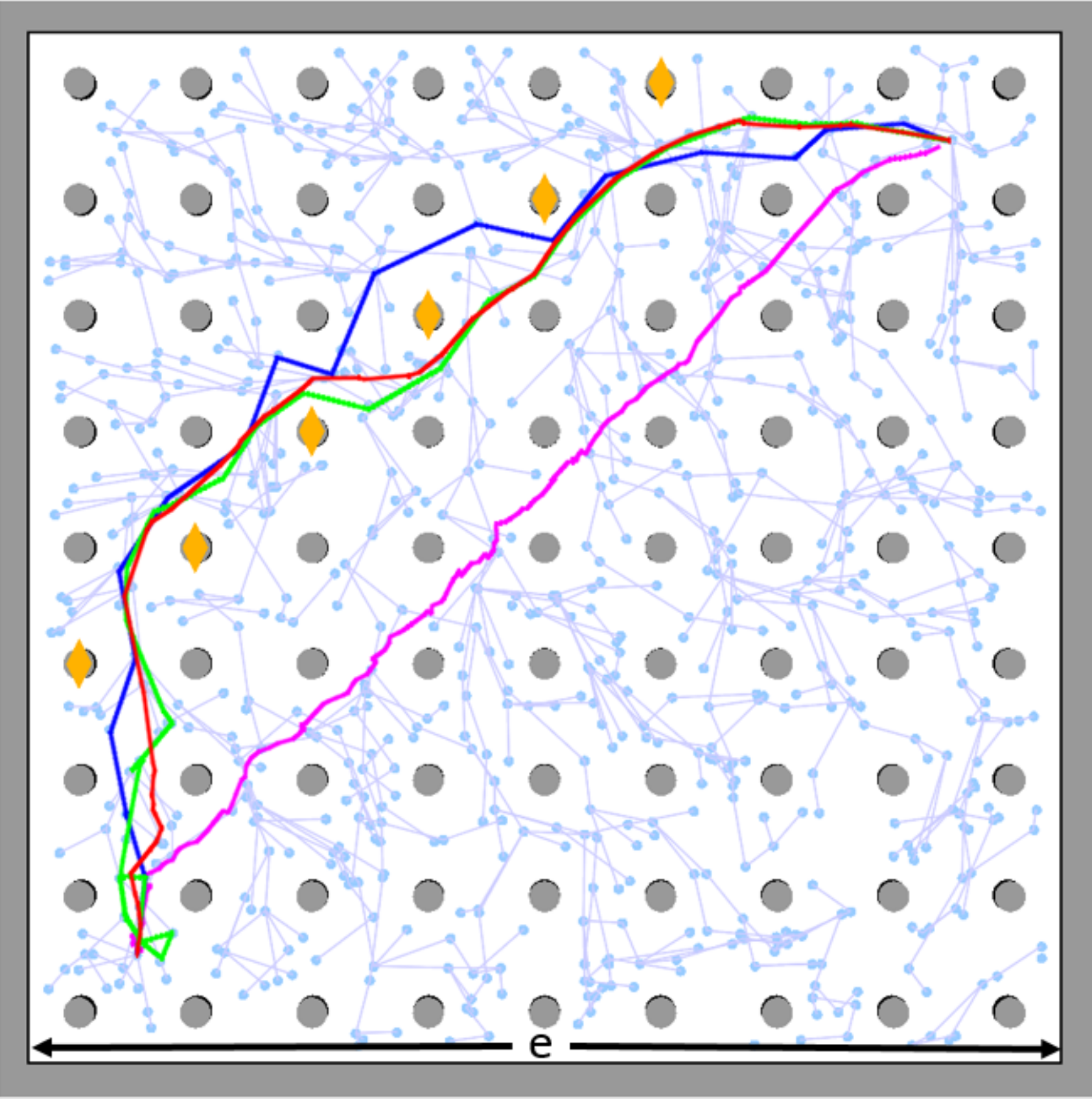}
		\caption{With 800 FIRM nodes}
	\end{subfigure}
	\caption{Execution trajectories for URM-POMCP (pink), FIRM (blue), OGR (green), and \ouralgo{} (red) on $\rnp{}_{\forest{}}(20,81)$ problem. The start and the goal states are on the left bottom and the right top of the map, respectively.
		While both OGR and \ouralgo{} take shortcuts instead of following FIRM's offline policy, OGR suffers from local minima near the start state.}
	\label{fig:nodedensity}
\end{figure}


\begin{figure}[t]
	\centering
		\includegraphics[clip,trim=2.0cm 0cm 2.0cm 0cm,
		height=4.2cm]{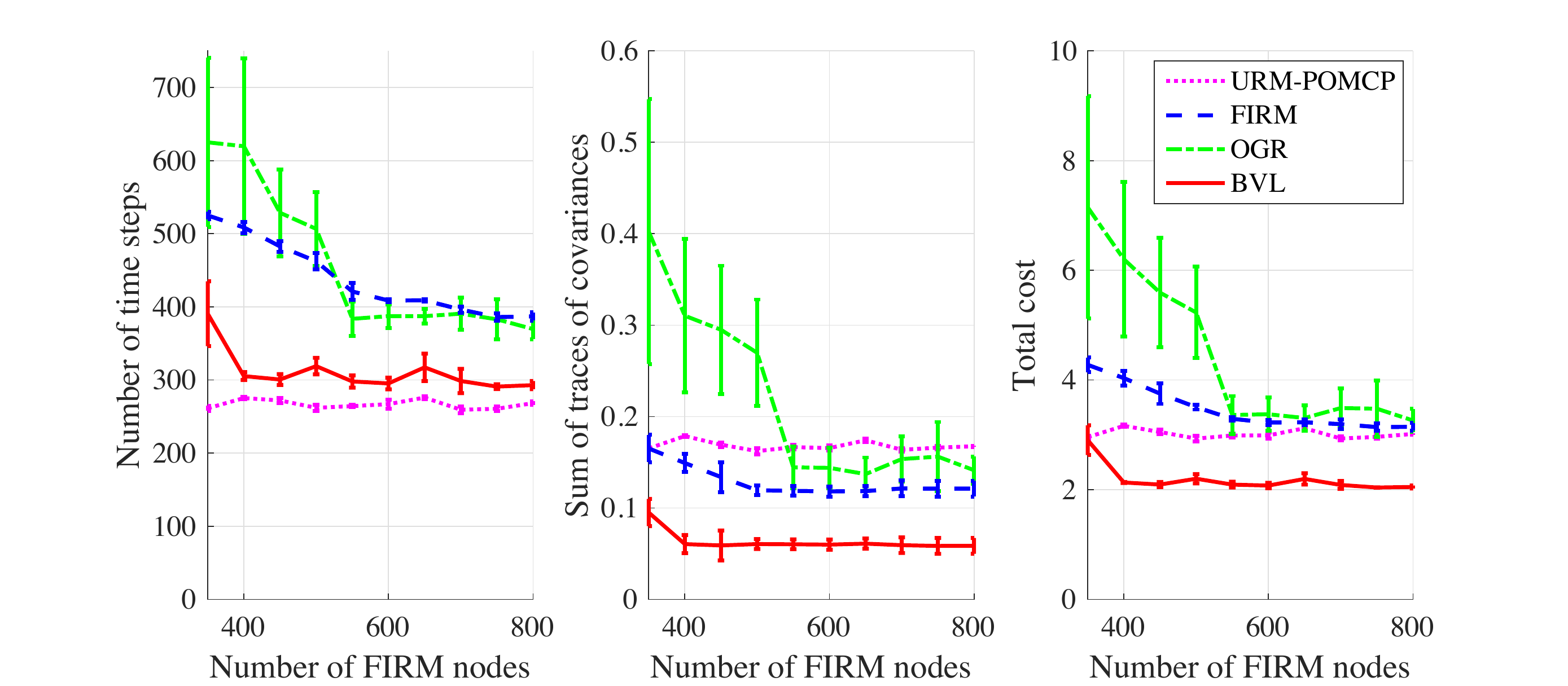}
		\caption{Plots for number of time steps, sum of traces of covariances, and total costs over different number of FIRM nodes.
			\ouralgo{} performs the best and is least affected by the FIRM node density.}
		\label{fig:plotnodedensity}
\end{figure}

The fundamental contribution of this method is to achieve policies that are closer to the globally optimal policies while reducing the risk of collisions over long horizons. 
To compare the optimality of the planners, we consider $\rnp{}_{\forest{}}(e,o)$ problems shown in Fig.~\ref{fig:nodedensity}, where $e$ represents the length of the environment and $o$ represents the number of obstacles.
We vary the density of the underlying belief graph to demonstrate its effect on the proposed method.

As can be seen in Fig.~\ref{fig:plotnodedensity}, the performance of the FIRM solution improves as the density of the underlying graph gets higher.
However, it will reach a maximum suboptimal bound due to its sampling-based nature (i.e., it requires stabilization of the belief to the stationary covariance of the graph nodes before leaving them).
In this complex environment, OGR with myopic online replanning frequently gets stuck at local minima, while it sometimes outperforms FIRM.
Its performance is brittle and subject to the coverage of the underlying belief graph.
%
In contrast, \ouralgo{} performs well even with a smaller number of nodes in the underlying graph.

While actions of the \ouralgo{} are selected from local controllers connecting to the nodes of the underlying belief graph, online belief tree search process fundamentally improves its behavior such that it is much less dependent on the density and coverage of the underlying graph. 
\ouralgo{} not only generates trajectories that are much closer to global optimum but also reduces the risk of collision over an infinite horizon.

\vspace{-4pt}
\section{Conclusion}

\begin{figure}[t]
        \begin{minipage}{0.56\columnwidth}
		\includegraphics[height=3.2cm]{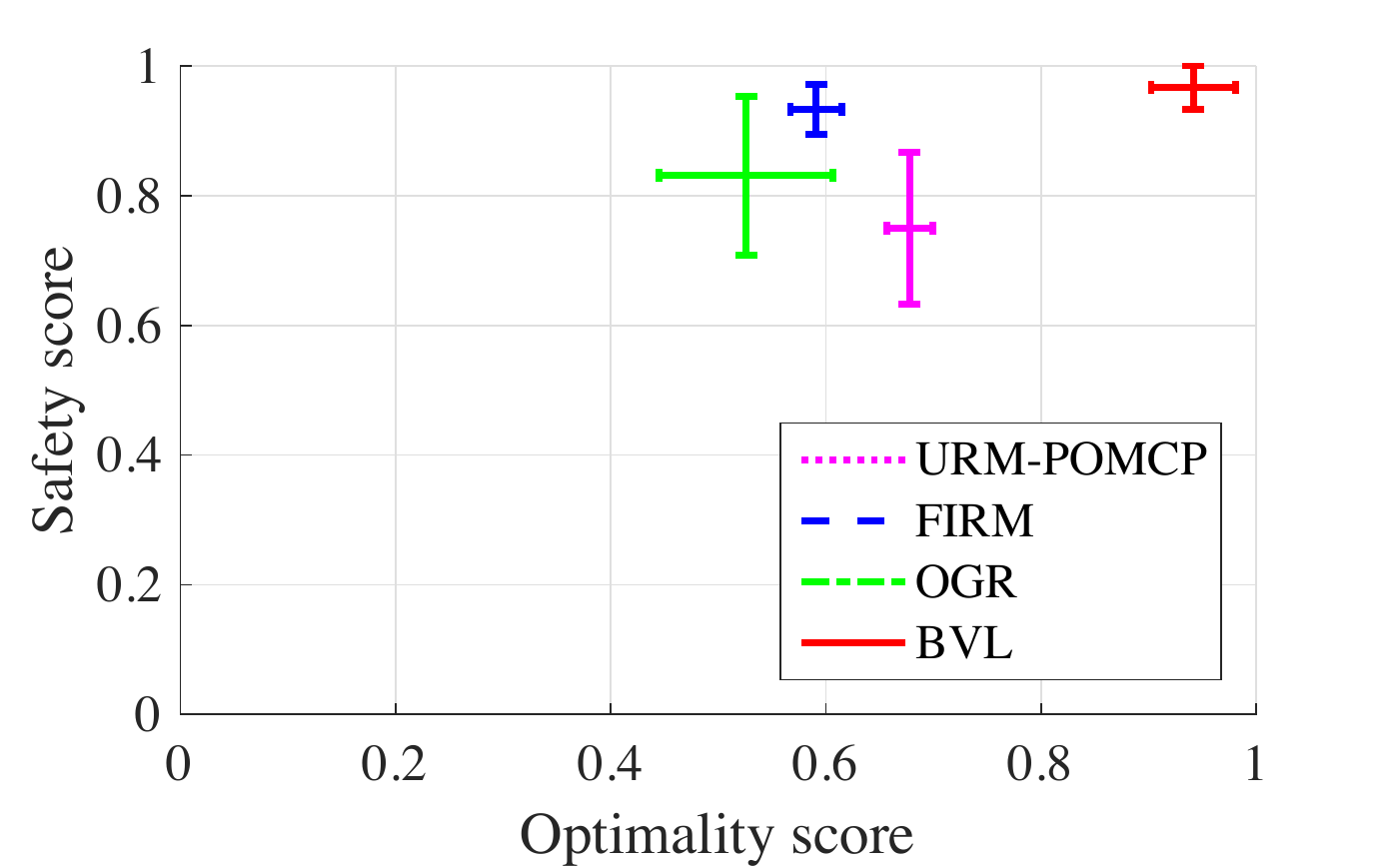}
		\end{minipage}
		\begin{minipage}{0.42\columnwidth}
		\caption{Plot for overall optimality and safety evaluation. 
		Optimality score is computed by dividing the minimal total cost over all runs by the individual total cost for $\rnp{}_{\forest{}}(e,o)$.
		Safety score is obtained by subtracting the probability of collision from $1$ for $\rnp{}_{\infotrap{}}(e,o)$.
		}
		\vspace{-0.4cm}
		\label{fig:plotoverall}
    \end{minipage}
\end{figure}

In this paper, we proposed \ouralgo{}, a novel bi-directional value learning algorithm that incorporates locally near-optimal forward search methods and globally safety-guaranteeing approximate long-range methods to solve challenging RAL-POMDP problems.
As shown in Fig.~\ref{fig:plotoverall}, \ouralgo{} provides better probabilistic safety guarantees than forward search methods (URM-POMCP) and is closer to the optimal performance than approximate long-range methods (FIRM).
It also shows more consistency in different environments compared to online graph-based rollout methods (OGR).

In future work, we will study the theoretical properties of this algorithm more rigorously and extend this work to more general and challenging robotic applications, such as mobile manipulation.
\rev{
We will also investigate another instance of \ouralgo{} using a heuristic search-based belief space planner
that can connect to the approximate global policy using multi-goal planning techniques.
}
\vspace{-5pt}

\vspace{-4pt}


%

%

%

\ifCLASSOPTIONcaptionsoff
  \newpage
\fi


\ifarxivfin
\IEEEtriggeratref{5}  
\fi


\bibliographystyle{IEEEtran}
\ifarxivfin
\bibliography{3_FIRMCP_WAFR}  

\begin{thebibliography}{10}
\providecommand{\url}[1]{#1}
\csname url@samestyle\endcsname
\providecommand{\newblock}{\relax}
\providecommand{\bibinfo}[2]{#2}
\providecommand{\BIBentrySTDinterwordspacing}{\spaceskip=0pt\relax}
\providecommand{\BIBentryALTinterwordstretchfactor}{4}
\providecommand{\BIBentryALTinterwordspacing}{\spaceskip=\fontdimen2\font plus
\BIBentryALTinterwordstretchfactor\fontdimen3\font minus
  \fontdimen4\font\relax}
\providecommand{\BIBforeignlanguage}[2]{{%
\expandafter\ifx\csname l@#1\endcsname\relax
\typeout{** WARNING: IEEEtran.bst: No hyphenation pattern has been}%
\typeout{** loaded for the language `#1'. Using the pattern for}%
\typeout{** the default language instead.}%
\else
\language=\csname l@#1\endcsname
\fi
#2}}
\providecommand{\BIBdecl}{\relax}
\BIBdecl

\bibitem{Kaelbling98}
L.~P. Kaelbling, M.~L. Littman, and A.~R. Cassandra, ``Planning and acting in
  partially observable stochastic domains,'' \emph{Artificial Intelligence},
  vol. 101, pp. 99--134, 1998.

\bibitem{kochenderfer2015decision}
M.~J. Kochenderfer, \emph{Decision making under uncertainty: theory and
  application}.\hskip 1em plus 0.5em minus 0.4em\relax MIT press, 2015.

\bibitem{Kurniawati08-SARSOP}
H.~Kurniawati, D.~Hsu, and W.~Lee, ``{SARSOP}: Efficient point-based {POMDP}
  planning by approximating optimally reachable belief spaces,'' in
  \emph{Proceedings of Robotics: Science and Systems}, 2008.

\bibitem{Pineau03}
J.~Pineau, G.~Gordon, and S.~Thrun, ``Point-based value iteration: An anytime
  algorithm for {POMDP}s,'' in \emph{International Joint Conference on
  Artificial Intelligence}, 2003, pp. 1025--1032.

\bibitem{silver2010monte}
D.~Silver and J.~Veness, ``Monte-carlo planning in large pomdps,'' in
  \emph{Advances in Neural Information Processing Systems}, 2010, pp.
  2164--2172.

\bibitem{gelly2011monte}
S.~Gelly and D.~Silver, ``Monte-carlo tree search and rapid action value
  estimation in computer go,'' \emph{Artificial Intelligence}, vol. 175,
  no.~11, pp. 1856--1875, 2011.

\bibitem{somani2013despot}
A.~Somani, N.~Ye, D.~Hsu, and W.~S. Lee, ``Despot: Online pomdp planning with
  regularization,'' in \emph{Advances in Neural Information Processing
  Systems}, 2013, pp. 1772--1780.

\bibitem{kurniawati2016online}
H.~Kurniawati and V.~Yadav, ``An online pomdp solver for uncertainty planning
  in dynamic environment,'' in \emph{International Symposium on Robotics
  Research}.\hskip 1em plus 0.5em minus 0.4em\relax Springer, 2016, pp.
  611--629.

\bibitem{Prentice09}
S.~Prentice and N.~Roy, ``The belief roadmap: Efficient planning in belief
  space by factoring the covariance,'' \emph{International Journal of Robotics
  Research}, vol.~28, no. 11-12, pp. 1448--1465, October 2009.

\bibitem{Ali14-IJRR}
A.~{Agha-mohammadi}, S.~Chakravorty, and N.~Amato, ``{FIRM}: Sampling-based
  feedback motion planning under motion uncertainty and imperfect
  measurements,'' \emph{International Journal of Robotics Research}, vol.~33,
  no.~2, pp. 268--304, 2014.

\bibitem{ruszczynski2010risk}
A.~Ruszczy{\'n}ski, ``Risk-averse dynamic programming for {Markov} decision
  processes,'' \emph{Mathematical programming}, vol. 125, no.~2, pp. 235--261,
  2010.

\bibitem{majumdar2017should}
A.~Majumdar and M.~Pavone, ``How should a robot assess risk? {Towards} an
  axiomatic theory of risk in robotics,'' \emph{arXiv preprint
  arXiv:1710.11040}, 2017.

\bibitem{Kavraki96}
L.~Kavraki, P.~{\v{S}}vestka, J.~Latombe, and M.~Overmars, ``Probabilistic
  roadmaps for path planning in high-dimensional configuration spaces,''
  \emph{IEEE Transactions on Robotics and Automation}, vol.~12, no.~4, pp.
  566--580, 1996.

\bibitem{watkins1992q}
C.~J. Watkins and P.~Dayan, ``{Q}-learning,'' \emph{Machine learning}, vol.~8,
  no. 3-4, pp. 279--292, 1992.

\bibitem{HiRiseMapLocalization}
Y.~Tao, J.~Muller, and W.~Poole, ``Automated localisation of mars rovers using
  co-registered hirise-ctx-hrsc orthorectified images and wide baseline navcam
  orthorectified mosaics,'' \emph{Icarus}, vol. 280, pp. 139--157, 2016.

\bibitem{MarsHeli}
B.~Balaram, T.~Canham, C.~Duncan, H.~F. Grip, W.~Johnson, J.~Maki, A.~Quon,
  R.~Stern, and D.~Zhu, ``Mars helicopter technology demonstrator,'' in
  \emph{AIAA Atmospheric Flight Mechanics Conference}, 2018.

\bibitem{Tamas04}
T.~Kalm{\'a}r-Nagy, R.~D'Andrea, and P.~Ganguly, ``Near-optimal dynamics
  trajectory generation and control of an omnidirectional vehicle,''
  \emph{Robotics and Autonomous Systems}, vol.~46, no.~1, pp. 47--64, 2004.

\bibitem{Ali14-RolloutFIRM-ICRA}
A.~{Agha-mohammadi}, S.~Agarwal, A.~Mahadevan, S.~Chakravorty, D.~Tomkins,
  J.~Denny, and N.~Amato, ``Robust online belief space planning in changing
  environments: Application to physical mobile robots,'' in \emph{IEEE
  International Conf. on Robotics and Automation}, 2014.

\bibitem{agha2018slap}
A.~{Agha-mohammadi}, S.~Agarwal, S.-K. Kim, S.~Chakravorty, and N.~M. Amato,
  ``{SLAP}: {S}imultaneous localization and planning under uncertainty via
  dynamic replanning in belief space,'' \emph{IEEE Transactions on Robotics},
  vol.~34, no.~5, pp. 1195--1214, 2018.

\end{thebibliography}
\fi
\end{document}